%% file: main.tex
\documentclass{article}
\usepackage{ijcai22}

\usepackage{pgfplots}
\DeclareUnicodeCharacter{2212}{−}
\usepgfplotslibrary{groupplots,dateplot}
\usetikzlibrary{patterns,shapes.arrows}
\pgfplotsset{compat=newest}
\usepackage{tikz}
\usetikzlibrary{shapes, arrows.meta, positioning}

\usepackage{times}

\usepackage{soul}
\usepackage{url}
\usepackage[hidelinks]{hyperref}
\usepackage[utf8]{inputenc}
\usepackage[small]{caption}
\usepackage{graphicx}
\usepackage{amsmath}
\usepackage{amsfonts}
\usepackage{booktabs}
\usepackage{algorithm}
\usepackage{algorithmic}
\usepackage{amsthm}

\usepgfplotslibrary{groupplots}
\usepgfplotslibrary{fillbetween}
\usetikzlibrary{arrows,decorations.pathmorphing,positioning,fit,trees,shapes,shadows,automata,calc} 
\usetikzlibrary{patterns,arrows,arrows.meta,calc,shapes,shadows,decorations.pathmorphing,decorations.pathreplacing,automata,shapes.multipart,positioning,shapes.geometric,fit,circuits,trees,shapes.gates.logic.US,fit, matrix,arrows.meta, quotes}
\usetikzlibrary{backgrounds,scopes}

\usepackage{wrapfig}

\usepackage{mathtools}
\usepackage{accents}

\newcommand{\defeq}{\vcentcolon=}

\newtheorem{theorem}{Theorem}

\newtheorem{lemma}{Lemma}

\title{Taylor-Lagrange Neural Ordinary Differential Equations:\\ Toward Fast Training and Evaluation of Neural ODEs}

\author{
Franck Djeumou\thanks{These authors contributed equally. \\ The code is available at \href{https://github.com/wuwushrek/TayLaNets}{https://github.com/wuwushrek/TayLaNets}.}\textsuperscript{\rm 1} \and
Cyrus Neary \footnotemark[1]\textsuperscript{\rm 1} \and
Eric Goubault\textsuperscript{\rm 2} \and
Sylvie Putot\textsuperscript{\rm 2} \and
Ufuk Topcu\textsuperscript{\rm 1} \\
\affiliations
\textsuperscript{1}The University of Texas at Austin, United States\\
\textsuperscript{2}LIX, CNRS, \'Ecole Polytechnique, Institut Polytechnique de Paris, France\\
\{fdjeumou, cneary, utopcu\}@utexas.edu, \; \{goubault, putot\}@lix.polytechnique.fr
}

\begin{document}

\input{macros}

\maketitle
\input{tex/00_abstract}
\input{tex/01_intro}
\input{tex/02_related_work}
\input{tex/03_background}
\input{tex/04_taylor_nets}
\input{tex/05_numerical_experiments}
\input{tex/06_supervised_classification}
\input{tex/07_density_estimation}
\input{tex/08_conclusions}

\section*{Acknowledgements}
This work was supported in part by ARL W911NF2020132, AFOSR FA9550-19-1-0005, and NSF 1646522.

% \clearpage

\bibliographystyle{named}
\bibliography{bibliography}

\onecolumn

\begin{center} % supp mat title
    \begin{LARGE}
        \textbf{Supplementary Material - Taylor-Lagrange Neural Ordinary Differential Equations: Toward Fast Learning and Evaluation of Neural ODEs}
    \end{LARGE}
\end{center}

\appendix
\input{tex/09_proofs}
\input{tex/10_experimental_details}

\end{document}

%% file: macros.tex
%%%%% dynamical systems

% Optimization
\newcommand{\regularizationCoef}{\lambda}
\newcommand{\regTerm}{\Lambda}
\newcommand{\penaltyCoef}{\mu}
\newcommand{\penaltyTerm}{U}

% System Dynamics
\newcommand{\state}{x}
\newcommand{\trueNextState}{y}
\newcommand{\stateDim}{n}
\newcommand{\stateSpace}{\mathbb{R}^{\stateDim}}
\newcommand{\dynFun}{f}
\newcommand{\controlInput}{u}
\newcommand{\controlSpace}{\mathcal{U}}
\newcommand{\trajectory}{\tau}
\newcommand{\controlDim}{m}
\newcommand{\timeval}{t}

\newcommand{\vectorField}{f}

% Numerical integrator
\newcommand{\numIntervals}{H}
\newcommand{\timestep}{\Delta t}
\newcommand{\expansionOrder}{p}
\newcommand{\taylorLagrange}{TL}
\newcommand{\taylor}{T}

\newcommand{\remainder}{\mathcal{R}}
\newcommand{\midpoint}{\Gamma}
\newcommand{\midpointTime}{\xi}
\newcommand{\midpointUnknownTerm}{\Bar{\midpoint}}

% Neural Networks
\newcommand{\nnParamsAll}{\Theta}
\newcommand{\nnParams}{\theta}
\newcommand{\nnParamsFixed}{\hat{\nnParams}}
\newcommand{\numParams}{k}

\newcommand{\nnParamsMidpoint}{\phi}
\newcommand{\nnParamsMidpointFixed}{\hat{\nnParamsMidpoint}}

% Neural ODE
\newcommand{\neuralODE}{\mathrm{NODE}}
\newcommand{\TLneuralODE}{\mathrm{TLNODE}}
\newcommand{\predictionTime}{T}
\newcommand{\initTime}{t_0}
\newcommand{\initState}{\state_{0}}

% Training
\newcommand{\predNextState}{\Gamma}
\newcommand{\loss}{\mathcal{L}}
\newcommand{\dataset}{\mathcal{D}}
\newcommand{\datasetFixed}{\dataset_{\nnParamsFixed}}
\newcommand{\datasetDim}{|\dataset|}
\newcommand{\numTrainingTrajectories}{i}
\newcommand{\trajectoryTimeHorizon}{T}
\newcommand{\rollout}{n_r}
\newcommand{\odesolve}{\mathrm{ODESolve}}

% Constraints
\newcommand{\inequalConstr}{\Psi}
\newcommand{\equalConstr}{\Phi}
\newcommand{\constrDomain}{\mathcal{C}}
\newcommand{\numInequalConstr}{l}
\newcommand{\numEqualConstr}{v}
\newcommand{\collocationPoints}{\Omega}
\newcommand{\numCollocationPoints}{|\collocationPoints|}
\newcommand{\augmentedLagrangian}{\mathcal{L}}
\newcommand{\constrPenalty}{\mu}
\newcommand{\lagrangeVar}{\lambda}
\newcommand{\lagrangeVarAll}{\Lambda}
\newcommand{\totalNumEqualConstr}{N_{\equalConstr}}
\newcommand{\totalNumInequalConstr}{N_{\inequalConstr}}

\newcommand{\collectionUnknownTerms}{G_{\nnParamsAll}}

% algorithm
\newcommand{\constrTol}{\epsilon}

% dynamics equations
\newcommand{\massMat}{M}
\newcommand{\corForce}{C}
\newcommand{\actuationForce}{\tau}
\newcommand{\contactForce}{\mathcal{F}}
\newcommand{\jacobian}{J}

%% file: tex/00_abstract.tex
\begin{abstract}
    Neural ordinary differential equations (NODEs) -- parametrizations of differential equations using neural networks -- have shown tremendous promise in learning models of unknown continuous-time dynamical systems from data.
    However, every forward evaluation of a NODE requires numerical integration of the neural network used to capture the system dynamics, making their training prohibitively expensive.
    Existing works rely on off-the-shelf adaptive step-size numerical integration schemes, which often require an excessive number of evaluations of the underlying dynamics network to obtain sufficient accuracy for training.
    By contrast, we accelerate the evaluation and the training of NODEs by proposing a data-driven approach to their numerical integration.
    The proposed Taylor-Lagrange NODEs (TL-NODEs) use a fixed-order Taylor expansion for numerical integration, while also learning to estimate the expansion's approximation error.
    As a result, the proposed approach achieves the same accuracy as adaptive step-size schemes while employing only low-order Taylor expansions, thus greatly reducing the computational cost necessary to integrate the NODE.
    A suite of numerical experiments, including modeling dynamical systems, image classification, and density estimation, demonstrate that TL-NODEs can be trained more than an order of magnitude faster than state-of-the-art approaches, without any loss in performance.
\end{abstract}

%% file: tex/01_intro.tex
\section{Introduction}

Neural ordinary differential equations (NODEs) have recently shown tremendous promise as a means to learn unknown continuous-time dynamical systems from trajectory data \cite{chen2018neural}.
By parametrizing differential equations as neural networks, as opposed to directly fitting the available trajectory data, NODEs provide compact representations of continuous-time systems that are memory-efficient and that are well understood; they allow the user to harness an existing wealth of knowledge from applied mathematics, physics, and engineering.
For example, recent works have used NODEs as a means to incorporate physics-based knowledge into the learning of dynamical systems~\cite{djeumou2021neural,menda2019structured,gupta2020structured,cranmer2020lagrangian,greydanus2019hamiltonian,finzi2020simplifying,zhong2021benchmarking}.
Furthermore, NODEs have been used to define continuous normalizing flows -- a class of invertible density models -- to learn complex probability distributions over data \cite{chen2018neural,grathwohl2018ffjord,mathieu2020riemannian,salman2018deep}.

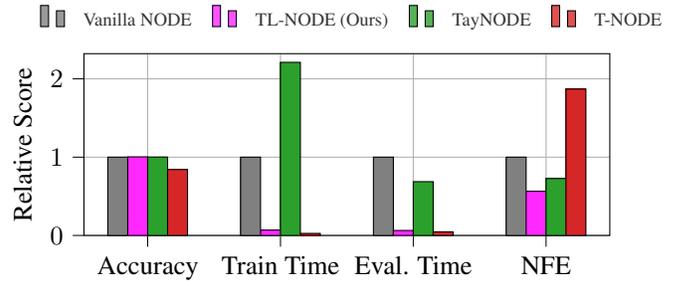
\begin{figure}[t]
    \centering
    \input{figs/intro_results_plot}
    \vspace{-2mm}
    \caption{The proposed TL-NODE (magenta) achieves evaluation and training times that are more than an order of magnitude faster than state-of-the-art methods, without compromising any accuracy. The plot illustrates the results of using TL-NODE for a classification task on the MNIST dataset \protect\cite{deng2012mnist}. All scores are relative to those obtained by a vanilla NODE (grey), which uses an adaptive timestep numerical integrator. We compare against TayNODE \protect\cite{kelly2020learning} (green), and T-NODE (red) which uses a Taylor expansion for integration without the proposed correction employed by TL-NODE. The number of function evaluations (NFE) measures the regularity of the learned NODE (lower is better).}
    \vspace{-5mm}
    \label{fig:intro_results_plot}
\end{figure} 

\begin{figure*}[t]
    \centering
    \input{figs/TLNODE_illustration}
    \caption{
    Taylor-Lagrange NODE (TL-NODE): an illustration of forward model evaluations. Given the system state \(\state_{\timeval_i}\) at time \(\timeval_{i}\), TL-NODE outputs a prediction of the state \(\hat{\state}_{\timeval_{i+1}}\) at future time \(\timeval_{i+1} = \timeval_{i} + \timestep\). The dynamics network (yellow) parametrizes the differential equation being modeled. We use a truncated Taylor expansion (brown) of the state dynamics \(\state_{\timeval}\) to predict the future state. A separate midpoint prediction network (blue) is trained to estimate the remainder of the expansion, which is used as a correction for the model's prediction.}
    \label{fig:TLNODE_illustration}
\end{figure*}
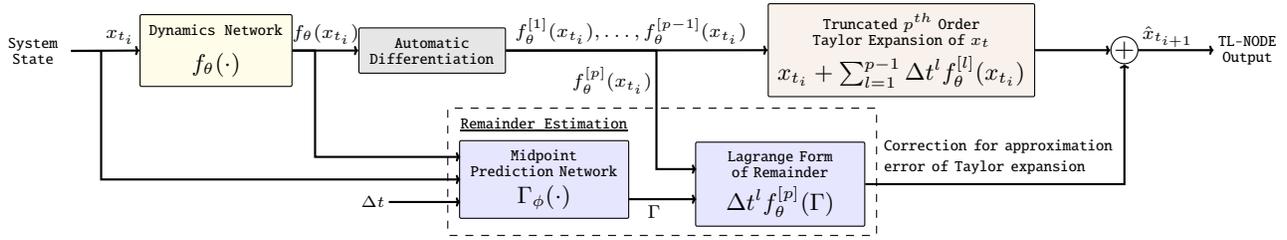 

However, the training of NODEs can become prohibitively expensive \cite{grathwohl2018ffjord,kelly2020learning,Finlay2020HowTT}.
In particular, every forward evaluation of the NODE requires the numerical integration of the underlying neural network parametrizing the system dynamics.
Existing methods for the training of NODEs use off-the-shelf adaptive step-size numerical integration schemes for this purpose.
However, in order to obtain sufficient accuracy, such integration schemes have been shown in practice to require an excessive number of evaluations of the underlying neural network.
Furthermore, the severity of this problem has been shown to grow as the training of the neural ODE progresses; while the neural ODE learns to fit the available data, it does not learn a representation of the dynamics that is easy to integrate \cite{Finlay2020HowTT}.
These computational issues render the training of neural ODEs on large datasets intractable, and they also prevent neural ODEs from being deployed in applications requiring repeated fast online predictions; such as for model-predictive control of robotic systems.

To address the above issues, we present Taylor-Lagrange NODEs (TL-NODEs); 
we use a truncated Taylor expansion of the underlying neural network to predict the future system state, and we train a separate neural network to correct this prediction according to the Lagrange form of the expansion's remainder.
By training the second corrector network, the approach significantly reduces the computational cost necessary for accurate numerical integration, while ensuring little-to-no-loss in the accuracy of the model.
Figure \ref{fig:TLNODE_illustration} illustrates the major components of TL-NODE, which are discussed below.

    \paragraph{(1) Taylor expansions for the numerical integration of the NODE.}
    In order to integrate the NODE, we use a fixed-order Taylor expansion of the dynamical system in time.
    We take advantage of Taylor-mode automatic differentiation to efficiently compute the higher-order terms of the expansion on a GPU \cite{Bettencourt2019TaylorModeAD}.
    By specifying the number of terms to include in the expansion, we ensure that only a fixed number of evaluations of the underlying dynamics network are required per training step. 
    
    \paragraph{(2) Correcting the expansion's approximation error.}
    We use the Lagrange form of the expansion's remainder to define a correction term, which greatly improves the accuracy of the predictions obtained from the truncated expansion.
    To estimate this approximation error efficiently, we propose to use a neural network to predict the so-called midpoint value -- a point near the center of the expansion at which the approximation error can be evaluated exactly.
    While learning this midpoint value may, in general, be as complex as learning the neural ODE itself,
    we derive explicit formulas for the midpoint using assumptions on the regularity of the dynamics.
    These expressions reduce the complexity of the learning problem; only one unknown term in the expression must be learned.
    We provide upper bounds on the error of the proposed Taylor-Lagrange expansion, in terms of the error in the predicted midpoint value and the order of the expansion.
    
We demonstrate the effectiveness of the proposed approach through a suite of numerical experiments.
The experimental tasks include the integration of known dynamics, learning to predict unknown dynamics, supervised classification, and density estimation.
Figure \ref{fig:intro_results_plot} illustrates the result of applying TL-NODE to a classification task.
Across all experiments we observe that the training and the evaluation of TL-NODEs is more than an order of magnitude faster than existing NODE methods, while incurring no loss in performance.

%% file: figs/intro_results_plot.tex
% This file was created with tikzplotlib v0.9.12.
\begin{tikzpicture}

\definecolor{color0}{RGB}{128,128,128}
\definecolor{color1}{RGB}{255,41,255}
\definecolor{color2}{rgb}{0.172549019607843,0.627450980392157,0.172549019607843}
\definecolor{color3}{rgb}{0.83921568627451,0.152941176470588,0.156862745098039}

\begin{axis}[
height=4cm,
width=\columnwidth,
legend cell align={left},
legend columns = 4,
legend style={
    /tikz/every even column/.append style={column sep=0.2cm},
    font=\scriptsize,
  fill opacity=0.8,
  draw opacity=1,
  text opacity=1,
  at={(-0.1,1.3)},
  anchor=north west,
  draw=none,
},
tick align=outside,
tick pos=left,
x grid style={white!69.0196078431373!black},
xmajorgrids,
xmin=-0.255, xmax=3.705,
xtick style={color=black},
xtick={0.225,1.225,2.225,3.225},
xticklabels={Accuracy,Train Time,Eval. Time,NFE},
y grid style={white!69.0196078431373!black},
ylabel={Relative Score},
ymajorgrids,
ymin=0, ymax=2.321484375,
ytick style={color=black}
]

\draw[draw=black,fill=color0] (axis cs:-0.075,0) rectangle (axis cs:0.075,1);
\addlegendimage{ybar,ybar legend,draw=black,fill=color0}
\addlegendentry{Vanilla NODE}

\draw[draw=black,fill=color0] (axis cs:0.925,0) rectangle (axis cs:1.075,1);
\draw[draw=black,fill=color0] (axis cs:1.925,0) rectangle (axis cs:2.075,1);
\draw[draw=black,fill=color0] (axis cs:2.925,0) rectangle (axis cs:3.075,1);
\draw[draw=black,fill=color1] (axis cs:0.075,0) rectangle (axis cs:0.225,1.00367834883008);
\addlegendimage{ybar,ybar legend,draw=black,fill=color1}
\addlegendentry{TL-NODE (Ours)}

\draw[draw=black,fill=color1] (axis cs:1.075,0) rectangle (axis cs:1.225,0.068359375);
\draw[draw=black,fill=color1] (axis cs:2.075,0) rectangle (axis cs:2.225,0.0638036809815951);
\draw[draw=black,fill=color1] (axis cs:3.075,0) rectangle (axis cs:3.225,0.563636363636364);
\draw[draw=black,fill=color2] (axis cs:0.225,0) rectangle (axis cs:0.375,1.00153264534587);
\addlegendimage{ybar,ybar legend,draw=black,fill=color2}
\addlegendentry{TayNODE}

\draw[draw=black,fill=color2] (axis cs:1.225,0) rectangle (axis cs:1.375,2.2109375);
\draw[draw=black,fill=color2] (axis cs:2.225,0) rectangle (axis cs:2.375,0.687116564417178);
\draw[draw=black,fill=color2] (axis cs:3.225,0) rectangle (axis cs:3.375,0.727272727272727);
\draw[draw=black,fill=color3] (axis cs:0.375,0) rectangle (axis cs:0.525,0.841524471237356);
\addlegendimage{ybar,ybar legend,draw=black,fill=color3}
\addlegendentry{T-NODE}

\draw[draw=black,fill=color3] (axis cs:1.375,0) rectangle (axis cs:1.525,0.02625);
\draw[draw=black,fill=color3] (axis cs:2.375,0) rectangle (axis cs:2.525,0.0460736196319018);
\draw[draw=black,fill=color3] (axis cs:3.375,0) rectangle (axis cs:3.525,1.87);

\end{axis}

\end{tikzpicture}

%% file: figs/TLNODE_illustration.tex
\begin{tikzpicture}[
    title/.style={font=\fontsize{6}{6}\color{black}\ttfamily},
    typetag/.style={rectangle, draw=black!50, font=\scriptsize\ttfamily, anchor=west},
    emptytag/.style={rectangle, draw=black!0, font=\scriptsize\ttfamily, anchor=west},
    plustag/.style={circle, draw=black!50, font=\scriptsize\ttfamily, anchor=west},
    line/.style={thick, draw},
    Arr/.style={-{Latex[length=1.5mm]}},
]

\node (origin) [emptytag] { };

\node (input) [emptytag, 
                title,
                inner xsep=0mm, 
                xshift=2mm, 
                yshift=0mm,
                text width=10mm,
                align=center,
                ] {System State};

%%%% Dynamics network
\node (dynNetTitle) [
        title, 
        anchor=west,
        right=20mm of origin.east,
        yshift=2.7mm,
        inner sep = 0mm,
        align=center,
        ] {Dynamics Network};

\node (dynNet) [
        below=2mm of dynNetTitle.south,
        inner sep = 0mm,
        align=center,
    ] {\small \(\vectorField_{\nnParams}(\cdot)\)};

\begin{scope}[on background layer]
    \node (dynNetBox) [
        draw=black,
        fill=yellow!10,
        rounded corners=0.2mm,
        fit={(dynNetTitle) (dynNet)}] {};
\end{scope}

%%%%% Automatic Differentiation
\node (autoDiffTitle) [
        title, 
        anchor=west,
        right=10mm of dynNetBox.east,
        yshift=0mm,
        inner sep = 0mm,
        text width=17mm,
        align=center,
        ] {Automatic Differentiation};

\begin{scope}[on background layer]
    \node (autoDiffBox) [
        draw=black,
        fill=black!10,
        rounded corners=0.2mm,
        fit={(autoDiffTitle)}] {};
\end{scope}

%%%% Taylor Expansion
\node (tlTitle) [
        title, 
        anchor=west,
        right=40mm of autoDiffBox.east,
        yshift=2.7mm,
        inner sep = 0mm,
        text width = 25mm,
        align=center,
        ] {Truncated \(\expansionOrder^{th}\) Order Taylor Expansion of \(\state_{\timeval}\)};

\node (tl) [
        below=1mm of tlTitle.south,
        inner sep = 0mm,
        align=center,
    ] {\small \(\state_{\timeval_i} + \sum\nolimits_{l=1}^{p-1} \timestep^l \vectorField_{\nnParams}^{[l]}(\state_{\timeval_i})\)};

\begin{scope}[on background layer]
    \node (tlBox) [
        draw=black,
        fill=brown!10,
        rounded corners=0.2mm,
        fit={(tlTitle) (tl)}] {};
\end{scope}

%%%% + thingy
\node (plus) [
        draw,
        anchor=west,
        right=10mm of tlBox.east,
        yshift=0mm,
        inner sep = 0mm,
        align=center,
        shape = circle,
        ] {\(+\)};
        
%%%% TLNODE output
\node (tlOutput) [
        title,
        anchor=west,
        right=10mm of plus.east,
        yshift=0mm,
        inner sep=0mm,
        inner xsep=0mm,
        align=center,
        shape = circle,
        text width=8mm,
        ] {TL-NODE Output};

%%%% Midpoint Prediction Network
\node (midpointNetTitle) [
        title, 
        anchor=west,
        right=-5mm of autoDiffBox.east,
        yshift=-15mm,
        inner sep = 0mm,
        text width = 20mm,
        align=center,
        ] {Midpoint Prediction Network};

\node (midpointNet) [
        below=1mm of midpointNetTitle.south,
        inner sep = 0mm,
        align=center,
    ] {\small \(\midpoint_{\nnParamsMidpoint}(\cdot)\)};

\begin{scope}[on background layer]
    \node (midpointNetBox) [
        draw=black,
        fill=blue!10,
        rounded corners=0.2mm,
        fit={(midpointNetTitle) (midpointNet)}] {};
\end{scope}

%%%% Remainder term
\node (remainderTitle) [
        title, 
        anchor=west,
        right=10mm of midpointNetBox.east,
        yshift=2mm,
        inner sep = 0mm,
        text width = 20mm,
        align=center,
        ] {Lagrange Form of Remainder};

\node (remainder) [
        below=1mm of remainderTitle.south,
        inner sep = 0mm,
        align=center,
    ] {\small \(\timestep^l \vectorField_{\nnParams}^{[p]}(\Gamma)\)};

\begin{scope}[on background layer]
    \node (remainderBox) [
        draw=black,
        fill=blue!10,
        rounded corners=0.2mm,
        fit={(remainderTitle) (remainder)}] {};
\end{scope}

%%%%%%% Remainder term enclosure

\node (remainderEnclosureTitle)[
        title, 
        anchor=west,
        above=0mm of midpointNetBox.north,
        xshift=0mm,
        yshift=1mm,
        inner sep = 0mm,
        inner ysep=0mm,
        inner xsep=0mm,
        align=left,
        ] {\underline{Remainder Estimation}};

\begin{scope}[on background layer]
    \node (remainderEnclosure) [
        draw=black,
        dashed,
        % fill=blue!10,
        % fill opacity = 0.5,
        inner xsep=1.5mm,
        rounded corners=0.2mm,
        fit={(remainderEnclosureTitle) (midpointNetBox) (remainderBox)}] {};
\end{scope}

% %%%%%%%%%%%% Connections

\draw [arrows={->[scale=0.5]}, thick] (input.east)+(0mm,0mm) -- node[font=\scriptsize\ttfamily, pos=1.0, above, xshift=3mm, yshift=-0.5mm, text width=1.5cm] {\(\state_{\timeval_i}\)} (dynNetBox.west);

\draw [arrows={->[scale=0.5]}, thick] (dynNetBox.east) -- node[font=\scriptsize\ttfamily, pos=0.9, above, xshift=-0.4mm, yshift=-0.5mm, text width=1.5cm] {$\vectorField_{\nnParams}(\state_{\timeval_i})$} (autoDiffBox.west);

\draw [arrows={->[scale=0.5]}, thick] (autoDiffBox.east) -- node[font=\scriptsize\ttfamily, pos=0.25, above, xshift=0.0mm, yshift=-0.5mm, text width=1.5cm]  {\(\vectorField^{[1]}_{\nnParams}(\state_{\timeval_i}), \ldots, \vectorField^{[p-1]}_{\nnParams}(\state_{\timeval_i})\)} (tlBox.west);

%%%% Inputs to midpoint network
\node (midpointNetBoxIn1) [left=0mm of midpointNetBox.west, xshift=0.65mm, yshift=3mm, inner xsep=0.0mm, inner ysep=0.0mm] {};
\node (midpointNetBoxIn2) [left=0mm of midpointNetBox.west, xshift=0.65mm, yshift=0mm, inner xsep=0.0mm, inner ysep=0.0mm] {};
\node (midpointNetBoxIn3) [left=0mm of midpointNetBox.west, xshift=0.65mm, yshift=-3mm, inner xsep=0.0mm, inner ysep=0.0mm] {};

\draw [arrows={->[scale=0.5]}, thick] (dynNetBox.east)+(0.0mm, 0.0mm) -- node[font=\scriptsize\ttfamily, text width=2cm, below, pos=1.8] {} ++(0.3cm,0) |- (midpointNetBoxIn1);

\draw [arrows={->[scale=0.5]}, thick] (input.east)+(0.0mm, 0.0mm) -- node[font=\scriptsize\ttfamily, text width=2cm, below, pos=1.8] {} ++(0.4cm,0) |- (midpointNetBoxIn2);

\node (deltaTIntoMidpointNet)[
        title,
        inner sep = 0mm,
        left=0mm of midpointNetBoxIn3,
        xshift=-10mm,
        ] {\(\timestep\)};

\draw [arrows={->[scale=0.5]}, thick] (deltaTIntoMidpointNet.east)+(0.5mm, 0.0mm) -- node[font=\scriptsize\ttfamily, text width=2cm, below, pos=1.8] {} ++(0.4cm,0) |- (midpointNetBoxIn3);

%%%% Inputs to remainder term
\node (remainderBoxIn1) [left=0mm of remainderBox.west, xshift=0.65mm, yshift=2mm, inner xsep=0.0mm, inner ysep=0.0mm] {};
\node (remainderBoxIn2) [left=0mm of remainderBox.west, xshift=0.65mm, yshift=-2mm, inner xsep=0.0mm, inner ysep=0.0mm] {};

\draw [arrows={->[scale=0.5]}, thick] (remainderBoxIn2.west)+(-9.0mm, 0.0mm) -- node[font=\scriptsize\ttfamily, text width=2cm, below, xshift=8mm, yshift=0.5mm] {\(\midpoint\)} (remainderBoxIn2);

\draw [arrows={->[scale=0.5]}, thick] (autoDiffBox.east)+(0.0mm, 0.0mm) -- node[font=\scriptsize\ttfamily, text width=2cm, right, yshift=-4mm, xshift=-2.5mm] {\(\vectorField^{[p]}_{\nnParams}(\state_{\timeval_i})\)} ++(2.0cm,0) |- (remainderBoxIn1);

%% Inputs to (plus)
\draw [arrows={->[scale=0.5]}, thick] (tlBox.east)+(0.0mm, 0.0mm) -- node[font=\scriptsize\ttfamily, text width=2cm, right, yshift=-4mm] {} (plus.west);

\draw [arrows={->[scale=0.5]}, thick] (remainderBox.east)+(0.0mm, 0.0mm) -- node[font=\scriptsize\ttfamily, text width=2cm, above, xshift=10mm, yshift=0mm, text width=35mm] {\tiny Correction for approximation error of Taylor expansion} ++(2.0cm,0) -| (plus.south);

%% Output arrow
\draw [arrows={->[scale=0.5]}, thick] (plus.east)+(0.0mm, 0.0mm) -- node[font=\scriptsize\ttfamily, text width=2cm, above, xshift=5.5mm, yshift=-1mm] {\(\hat{\state}_{\timeval_{i+1}}\)} (tlOutput);
        
\end{tikzpicture}

%% file: tex/02_related_work.tex
\paragraph{Related Work.}
\label{sec:related_work}

Similarly to our work, a number of papers study how to reduce the computational cost of training neural ODEs.
After \cite{chen2018neural} initially presented the neural ODE, \cite{grathwohl2018ffjord} proposed a stochastic estimator of the likelihood to reduce computational cost when using neural ODEs for continuous normalizing flows.
\cite{kelly2020learning,Finlay2020HowTT,pmlr-v139-pal21a} propose additional regularization terms to learn neural ODEs that are easy to integrate.
\cite{ghosh2020steer} propose to regularize the neural ODE by randomly sampling the end time of the ODE during training.
However, all of these works use off-the-shelf numerical integration algorithms for the forward evaluation of the NODE.
By contrast, our work suggests a novel data-driven integration scheme, resulting in training times that are an order of magnitude faster than the current state-of-the-art. 

Meanwhile, \cite{Poli2020HypersolversTF} also suggest training an additional corrector neural network to speed up the numerical integration of NODEs.
However, they do not present a technique that is able to apply such a corrector network during the training of the NODE.
By contrast, we propose algorithms for the simultaneous training of the dynamics network and the remainder estimation network; this simultaneous training results not only in a speedup of the NODE evaluations, but also in a speedup of the NODE's training.
Furthermore, we propose a method to simplify the learning of the remainder term by taking advantage of regularity assumptions on the system dynamics.
This simplification leads to more efficient and generalizable learning of the correction term.

%% file: tex/03_background.tex
\section{Background}
\label{sec:preliminaries}
We begin by introducing necessary background on ordinary differential equations (ODEs) and on neural ODEs.

\paragraph{Ordinary Differential Equations (ODEs).} 
Let the  function $\vectorField (\state, \timeval) : \stateSpace \times \mathbb{R}_+ \mapsto \mathbb{R}^\stateDim$ be \emph{Lipschitz-continuous} in $\state$ and $\timeval$. An ordinary differential equation specifies the instantaneous change of a vector-valued signal of time $\state : \mathbb{R_+} \mapsto \stateSpace$.
\begin{align}
    \dot{\state}(\timeval) = \vectorField(\state(\timeval), \timeval)\label{eq:dynamics-formulation}
\end{align}

We note that in general, the explicit dependence of \(\vectorField(\state, \timeval)\) on \(\timeval\) can be removed by adding a dimension to the state variable \(\state\).
As such, throughout the remainder of the paper, we consider autonomous systems of the form \(\dot{\state}(\timeval) = \vectorField(\state(\timeval))\).
Furthermore, for notational simplicity we use the subscript notation \(\state_{\timeval}\) in place of \(\state(\timeval)\).

Given some initial state $\initState \defeq \state_{t_0} \in \stateSpace$ at initial time $\timeval_0 \geq 0$, we wish to compute a solution to \eqref{eq:dynamics-formulation}.
In this work, we are specifically interested in predicting the value of the state $\state_{\timeval}$, at an arbitrary future point in time $\predictionTime \geq \initTime$.
The value of \(\state_{\predictionTime}\) can be found by \textit{integrating} or \textit{solving} the ODE: \(\state_{\predictionTime} = \initState + \int_{\initTime}^{\predictionTime} \vectorField (\state_{s}) \,ds.\)

\paragraph{Neural Ordinary Differential Equations (NODEs).}
Neural ODEs (NODEs) are a class of deep learning models that use a neural network \(\vectorField_{\nnParams}(\state)\) to parametrize an ODE, and that subsequently use numerical integration algorithms to evaluate the model's outputs.
More specifically, given a NODE input tuple \((\initState, \initTime, \predictionTime)\) consisting of an initial state \(\initState\), an initial time \(\initTime\), and a prediction time \(\predictionTime\), the output \(\hat{\state}_{\predictionTime} \defeq \neuralODE_{\nnParams}(\initState,  \initTime, \predictionTime)\) of the NODE with parameter vector \(\nnParams\) is given by \(\mathrm{ODESolve}(\vectorField_{\nnParams}, \initState, \initTime, \predictionTime) \approx \initState + \int_{\initTime}^{\predictionTime} \vectorField_{\nnParams}(\state_{s})ds\).
Here, \(\mathrm{ODESolve}(\vectorField_{\nnParams}, \initState, \initTime, \predictionTime)\) is a numerical approximation of the solution to the dynamics parametrized by \(\vectorField_{\nnParams}(\state)\).

We note that the particular algorithm used in place of \(\mathrm{ODESolve}(\cdot)\) influences the model's accuracy, the computational cost of forward evaluations of the model, and the computational cost of training the model.
Existing implementations of neural ODEs have typically relied on existing adaptive-step numerical integration algorithms for this purpose.
By contrast, this work proposes the use of a novel Taylor expansion based method that uses data-drive estimations of the expansion's truncation error to evaluate and train neural ODEs efficiently and accurately.

%% file: tex/04_taylor_nets.tex
\section{Taylor-Lagrange Neural Ordinary Differential Equations (TL-NODE)}

In this section we propose TL-NODEs for the efficient and accurate training and evaluation of neural ODEs.

\subsection{Data-Driven Taylor-Lagrange Numerical Integration of the NODE Dynamics}

\label{sec:TL_numerical_integration_of_neural_ode}

Our objective is to efficiently and accurately evaluate neural ODEs through numerical integration of \(\vectorField_{\nnParams}(\cdot)\).
Toward this end, we propose to make direct use of the Taylor-Lagrange expansion of \(\state_{t}\) in time.
Taylor expansions are techniques used to approximate a function near a particular expansion point using polynomials.
We use the term Taylor-Lagrange expansion to refer to the truncated Taylor expansion including the Lagrange form of the remainder.
In order to obtain highly accurate expansions of \( \state_{t}\) without requiring an excessive number of evaluations of \(\vectorField_{\nnParams}(\cdot)\), we propose to train a separate neural network with parameters \(\nnParamsMidpoint\) to estimate the remainder term in the expansion.
We give a step-by-step description of the proposed methodology below.

\paragraph{Partitioning the prediction interval.}
We begin by partitioning the prediction interval \([\initTime, \predictionTime]\) into a collection of \(\numIntervals\) sub-intervals \([t_i, t_{i+1}]\).
For notational convenience, we assume the value of the integration time step is fixed, i.e. \(\timestep = t_{i+1} - t_{i}\) for all \(i = 0, \ldots, \numIntervals - 1\).
Given the value of the state \(\state_{t_{i}}\) at time \(t_i\), we compute an approximation of the value of the state at the next time step \(\hat{\state}_{t_{i+1}}\) via a \(p^{th}\) order Taylor-Lagrange expansion of \(\state_{t}\) about time \(\timeval_{i}\), which we denote by \(\taylorLagrange^{\timestep}_{\nnParams, \nnParamsMidpoint}(\state_{\timeval_i})\).
The output \(\TLneuralODE_{\nnParams, \nnParamsMidpoint}(\initState, \initTime, \predictionTime)\) of the neural ODE is computed by iteratively using \(\taylorLagrange^{\timestep}_{\nnParams, \nnParamsMidpoint}(\cdot)\) to approximate the integral of \(\vectorField_{\nnParams}(\cdot)\) over each sub-interval.

\paragraph{Expressing \(\boldsymbol{\taylorLagrange^{\timestep}_{\nnParams, \nnParamsMidpoint}(\cdot)}\) in Terms of \(\boldsymbol{\vectorField_{\nnParams}(\cdot)}\).}
Note that because \(\vectorField_{\nnParams}(\state)\) estimates the time derivative of the system state \(\state\), the Taylor-Lagrange expansion of \(\state_{t}\) may be expressed in terms of the \textit{Taylor coefficients} \(\vectorField_{\nnParams}^{[l]}(\state)\), which are recursively defined through the equations \(\vectorField^{[1]}_{\nnParams}(\state) = \vectorField_{\nnParams}(\state)\) and \(\vectorField^{[l+1]}_{\nnParams}(\state) = \frac{1}{l+1} [\frac{\partial \vectorField^{[l]}_{\nnParams}}{\partial \state} \vectorField_{\nnParams}](\state)\). 
Equation \eqref{eq:taylor_lagrange_expansion} accordingly presents the Taylor-Lagrange expansion of \(\state_{\timeval}\) about the point in time \(\timeval_{i}\), evaluated at the future time \(\timeval_{i+1} = \timeval_{i} + \timestep\).
\begin{equation}
\begin{split}
    \taylorLagrange_{\nnParams, \nnParamsMidpoint}^{\timestep}(\state_{\timeval_i}) \defeq \state_{\timeval_i} + \sum\nolimits_{l=1}^{p-1} & \timestep^l \vectorField_{\nnParams}^{[l]}(\state_{\timeval_i})
    \\ & + \remainder_{\nnParamsMidpoint}(\vectorField_{\nnParams}, \state_{\timeval_i}, \timestep)
\end{split}
\label{eq:taylor_lagrange_expansion}
\end{equation}
The first two terms on the right hand side of \eqref{eq:taylor_lagrange_expansion} make up the \textit{truncated} Taylor expansion of \(\state_{\timeval}\), while \(\remainder_{\nnParamsMidpoint}(\vectorField_{\nnParams}, \state_{\timeval_i}, \timestep)\) denotes an estimation of the \textit{remainder} of this truncated expansion.
More specifically, \(\remainder_{\nnParamsMidpoint}(\vectorField_{\nnParams}, \state_{\timeval_i}, \timestep)\) estimates the approximation error of the \(p^{th}\) order expansion; if we could known this value exactly, \eqref{eq:taylor_lagrange_expansion} would provide an exact evaluation of the integral \(\state_{t_i} + \int_{t_i}^{t_{i+1}} \vectorField_{\nnParams}(\state_{s}) ds\).
Below we propose a methodology to learn to accurately estimate the value of this remainder term, given \(\vectorField_{\nnParams}\), \(\state_{\timeval_i}\), and \(\timestep\) as inputs. 

\paragraph{Estimating the Remainder Term \(\boldsymbol{\remainder_{\nnParamsMidpoint}(\cdot)}\).}
To obtain accurate and generalizable estimations of the remainder term \(\remainder_{\nnParamsMidpoint}(\vectorField_{\nnParams}, \state_{\timeval_i}, \timestep)\), we begin by using Taylor's theorem to express it as \(\remainder_{\nnParamsMidpoint}(\vectorField_{\nnParams}, \state_{\timeval_i}, \timestep) = \vectorField_{\nnParams}^{[p]}(\midpoint)\).
Here, \(\midpoint \in \stateSpace\) denotes the \textit{midpoint} of the Taylor-Lagrange expansion.
More specifically, there exists some point in time \(\midpointTime\) with \(\timeval_{i} \leq \midpointTime \leq \timeval_{i+1}\) such that when we define \(\midpoint \defeq \state_{\midpointTime}\), then \(\timestep^p \vectorField_{\nnParams}^{[p]}(\midpoint)\) provides the exact value of the approximation error of the expansion.

While no closed form expression for the midpoint \(\midpoint\) exists, we propose to learn to predict its value given the state \(\state_{\timeval_i}\) and the time step \(\timestep\).
Learning to predict \(\midpoint\) directly as a function of these inputs is a challenging problem in general.
We instead propose to use the result of Theorem \ref{thm:expression-midpoint}, which provides a closed-form expression for \(\midpoint\) in terms of some unknown term \(\midpointUnknownTerm \in \stateSpace\).
By taking advantage of this expression for \(\midpoint\), we greatly simplify the task of learning to predict its value.

\begin{theorem}[Simplified Midpoint Expression]\label{thm:expression-midpoint}
    If \(f_{\theta}\) is a Lipschitz-continuous function, then there exists a function \(\bar{\Gamma} : \mathbb{R}^{n} \times \mathbb{R}_{+} \to \mathbb{R}^{n \times n}\) such that the midpoint value \(\Gamma\) of the Taylor-Lagrange expansion \(TL^{\Delta t}_{\theta, \phi}(x_{t_i})\) is related to \(x_{t_i}\), \(f_{\theta}(x_{t_i})\), and \(\bar{\Gamma}(x_{t_i}, \Delta t)\) through
    \begin{equation}
        \Gamma = x_{t_i} + \bar{\Gamma}(x_{t_i}, \Delta t) \odot  f_{\theta}(x_{t_i}),
        \label{eq:midpoint-simple}
    \end{equation}
    where \(\odot\) denotes matrix-vector multiplication.
\end{theorem}

Theorem~\ref{thm:expression-midpoint} is obtained using tools from interval Taylor-Langrange based reachability analysis \cite{djeumou2021fly}.
More specifically, the equation given in \eqref{eq:midpoint-simple} is derived from explicit formulas for the so-called \emph{apriori} enclosure -- a set derived from the local Lipschitzness of \(\vectorField_{\nnParams}\) that is guaranteed to contain the value of \(\taylorLagrange^{\timestep}_{\nnParams, \nnParamsMidpoint}(\state_{\timeval_i})\).
A proof of Theorem~\ref{thm:expression-midpoint} is provided in Appendix \ref{sec:appendix_proof_thm_1}.
We also prove that for linear dynamics, \(\midpointUnknownTerm\) does not depend on \(\state_{\timeval_i}\).

\paragraph{Estimating the Midpoint Value.}
Given the result of Theorem \ref{thm:expression-midpoint}, we propose to parameterize the unknown function \(\midpointUnknownTerm_{\nnParamsMidpoint}(\cdot)\) using a neural network with parameters \(\nnParamsMidpoint\).
For notational simplicity, we use \(\midpoint_{\nnParamsMidpoint}(\state, \timestep)\) to denote the value of the right hand side of \eqref{eq:midpoint-simple} when \(\midpointUnknownTerm(\cdot)\) is approximated by \(\midpointUnknownTerm_{\nnParamsMidpoint}(\cdot)\).
Given the predicted midpoint value \(\midpoint_{\nnParamsMidpoint}(\state, \timestep)\), we estimate the remainder term of the \(p^{th}\) order Taylor-Lagrange expansion \(\taylorLagrange^{\timestep}_{\nnParams, \nnParamsMidpoint}(\state_{\timeval_i})\) as \(\remainder_{\nnParamsMidpoint}(\vectorField_{\nnParams}, \state_{\timeval_i}, \timestep) \approx \timestep^{p}\vectorField_{\nnParams}^{[p]}(\midpoint_{\nnParamsMidpoint}(\state_{\timeval_i}, \timestep))\).

\paragraph{The Proposed TL-NODE Evaluation Algorithm.}
Algorithm \ref{alg:evaluation_alg} summarizes the proposed approach for the numerical evaluation of neural ODEs. 
In lines \(1\) and \(2\), the prediction interval \([\initTime, \predictionTime]\) is broken into \(\numIntervals\) sub-intervals.
The for loop in lines \(3-6\) iterates over these sub-intervals, and uses the midpoint prediction network \(\midpoint_{\nnParamsMidpoint}(\cdot)\) to estimate the midpoint value (line \(4\)), before using this estimate to approximate the state value \(\hat{\state}_{\timeval_{i+1}}\) at the end of the sub-interval (line 5).

\input{algs/evaluation_alg}

\paragraph{Bounding the Error of the TL-NODE Evaluation Algorithm.}
Given a fixed dynamics function \(\vectorField_{\nnParams}(\cdot)\), we seek to bound the error on a \(p^{th}\) order Taylor-Lagrange expansion which uses a learned midpoint value predictor \(\midpoint_{\nnParamsMidpoint}(\cdot)\) to estimate the expansion's remainder \(\remainder_{\nnParamsMidpoint}(\cdot)\).
Such an error bounds straightforwardly depends on how well \(\midpoint_{\nnParamsMidpoint}(\cdot)\) approximates the true midpoint \(\midpoint\), as described in Theorem \ref{thm-approximation-error}.
A proof of Theorem~\ref{thm-approximation-error} is provided in Appendix \ref{sec:appendix_proof_thm_2}.

\begin{theorem}[Integration Accuracy]\label{thm-approximation-error}
    If the learned midpoint predictor \(\midpoint_{\nnParamsMidpoint}(\cdot)\) is a \(\mathcal{O}(\eta)\) approximator to the midpoint \(\midpoint\) of the Taylor-Lagrange expansion of \(\taylorLagrange^{\timestep}_{\nnParams, \nnParamsMidpoint}(\state_{\timeval_i})\), then \(\| \state_{\timeval_{i+1}} - \taylorLagrange^{\timestep}_{\nnParams, \nnParamsMidpoint}(\state_{\timeval_i}) \| \leq c \eta \timestep^{p}\) for some \(c > 0\) that depends on \(\vectorField_{\nnParams}\).
\end{theorem}   

\paragraph{A Note on the Evaluating the Taylor Coefficients.}
The Taylor coefficients \(\vectorField_{\nnParams}^{[1]}(\cdot), \ldots, \vectorField_{\nnParams}^{[p]}(\cdot)\) can in principle be evaluated using repeated application of forward-mode automatic differentiation to iteratively compute the Jacobian-vector products \([\frac{\partial \vectorField^{[l]}_{\nnParams}}{\partial x} \vectorField_{\nnParams}](\state)\). 
However, doing so would incur a time cost of \(\mathcal{O}(exp(\expansionOrder))\).
We instead use \textit{Taylor mode} automatic differentiation, which computes the first \(\expansionOrder\) Taylor coefficients \(\vectorField^{[1]}_{\nnParams}(\cdot), \ldots, \vectorField^{[p]}_{\nnParams}(\cdot)\) in a single pass, with a time cost of only $\mathcal{O}(p^2)$ or of $\mathcal{O}(p \log p)$, depending on the underlying operations involved \cite{Griewank2000EvaluatingD,Bettencourt2019TaylorModeAD,kelly2020learning}.

\subsection{Training Taylor-Lagrange Neural Ordinary Differential Equations}
\label{sec:training_tlnode}

Given a training dataset \(\dataset\), we wish to train both components of the TL-NODE: the dynamics network \(\vectorField_{\nnParams}(\cdot)\) and the midpoint prediction network \(\midpoint_{\nnParamsMidpoint}(\cdot)\).
% To do so, we propose a training algorithm that iterates between individually training these two components.
To do so, we propose an algorithm that alternates between training each of the components via stochastic gradient descent while keeping the parameters of the other component fixed.
We assume that each datapoint within the dataset \(\dataset = \{(\initState^j, \initTime^j, \predictionTime^j, \trueNextState^j)\}_{j=1}^{|\dataset|}\) is comprised of an initial state \(\state^j\), an initial time \(\initTime^j\), a prediction time \(\predictionTime^j\), and a labeled output value \(\trueNextState^j\).

\paragraph{Training the Dynamics Network \(\boldsymbol{\vectorField_{\nnParams}(\cdot)}\).}

We begin by holding the parameter vector \(\nnParamsMidpoint\) of the midpoint prediction network to some fixed value \(\nnParamsMidpointFixed\), and training the dynamics network \(\vectorField_{\nnParams}(\cdot)\) by solving the optimization problem \eqref{eq:loss_opt_problem} via stochastic gradient descent.
\begin{equation}
\begin{split}
    \min_{\nnParams} \sum_{(\initState^j, \initTime^j, \predictionTime^j, \trueNextState^j) \in \dataset}& \Bigl[
    \loss(\TLneuralODE_{\nnParams, \nnParamsMidpointFixed}(\initState^j, \initTime^j, \predictionTime^j), \trueNextState^j) \Bigr. \\
    & \Bigl. + \regularizationCoef \sum_{i=0}^{\numIntervals - 1} || \timestep^p \vectorField_{\nnParams}^{[p]}(\midpoint_{\nnParamsMidpointFixed}(\hat{\state}_{\timeval_i}, \timestep)) ||^2
    \Bigr]
\end{split}   
\label{eq:loss_opt_problem}
\end{equation}
Here, \(\loss(\cdot)\) is any differentiable loss function and \(\hat{\state}_{\timeval_i}\) denotes integrator-estimated intermediate state at time \(\timeval_i\).

The summation in the second line of \eqref{eq:loss_opt_problem} measures the magnitude of the remainder terms \(\remainder_{\nnParamsMidpoint}(\cdot)\) of the truncated Taylor-Lagrange expansions used for numerical integration.
We may interpret this penalty term as having two purposes.
Firstly, it acts as a regularizer that penalizes the higher-order derivatives of \(\vectorField_{\nnParams}(\cdot)\) during training.
Intuitively, by penalizing these higher order derivatives we encourage solutions that fit the data while also remaining as simple as possible.
Secondly, the penalty term prevents the TL-NODE from using \(\remainder_{\nnParamsMidpoint}(\cdot)\) to overfit the training data.
By ensuring that the remainder term of the Taylor-Lagrange expansion remains small during training, we learn a dynamics function \(\vectorField_{\nnParams}(\cdot)\) whose truncated expansions fit the training data as well as possible, while using \(\remainder_{\nnParamsMidpoint}(\cdot)\) only for small corrections.

\paragraph{Training the Midpoint Prediction Network \(\boldsymbol{\midpoint_{\nnParamsMidpoint}(\cdot)}\).}

Recall that the midpoint prediction network \(\midpoint_{\nnParamsMidpoint}(\cdot)\) plays a crucial role in accurately integrating the dynamics specified by \(\vectorField_{\nnParams}(\cdot)\).
So, as the parameters of the dynamics network \(\vectorField_{\nnParams}(\cdot)\) are updated throughout training, our estimates of \(\midpoint\), the midpoint  of the Taylor-Lagrange expansion of \(\vectorField_{\nnParams}(\cdot)\), should be updated accordingly.
We thus propose to occasionally freeze the parameters of the dynamics network \(\nnParamsFixed\) in order to train \(\midpoint_{\nnParamsMidpoint}(\cdot)\).

After fixing \(\nnParamsFixed\), we begin by generating a small dataset \(\dataset_{\nnParamsFixed}\).
The datapoints of \(\dataset_{\nnParamsFixed}\) correspond to solutions of the ODE encoded by the fixed dynamics network \(\vectorField_{\nnParamsFixed}(\cdot)\).
That is, for each \((\initState, \initTime, \predictionTime, \trueNextState) \in \dataset_{\nnParamsFixed}\) the output label \(\trueNextState\) is given by \(\mathrm{ODESolve}(\vectorField_{\nnParamsFixed}, \initState, \initTime, \predictionTime)\), where \(\mathrm{ODESolve}(\cdot)\) is a highly accurate adaptive-step ODE solver.
Once the dataset \(\dataset_{\nnParamsFixed}\) has been generated, we train \(\midpoint_{\nnParamsMidpoint}(\cdot)\) by using stochastic gradient descent to solve the optimization problem \eqref{eq:midpoint_loss_opt_problem}.
\begin{equation}
\begin{split}
    \min_{\nnParamsMidpoint} \sum_{(\initState^j, \initTime^j, \predictionTime^j, \trueNextState^j) \in \datasetFixed} &
    || \TLneuralODE_{\nnParamsFixed, \nnParamsMidpoint}(\initState^j, \initTime^j, \predictionTime^j) - \trueNextState^j||^2
    % & \Bigl. + \regularizationCoef * \regTerm(\state^j, \initTime^j, \predictionTime^j)
    % \Bigr]
\end{split}   
\label{eq:midpoint_loss_opt_problem}
\end{equation}
\paragraph{The Proposed TL-NODE Training Algorithm.}
Algorithm \ref{alg:training_algorithm} details the proposed training procedure.
Throughout training we alternate between the following two subroutines: (lines \(3\)-\(4\)) fix \(\nnParamsMidpoint\) and take \(N_{\nnParams}\) stochastic gradient descent steps to train the dynamics network \(\vectorField_{\nnParams}(\cdot)\) according to \eqref{eq:loss_opt_problem}, (lines \(5\)-\(7\)) fix \(\nnParams\) and take \(N_{\nnParamsMidpoint}\) stochastic gradient descent steps to train the midpoint prediction network \(\midpoint_{\nnParamsMidpoint}(\cdot)\) according to \eqref{eq:midpoint_loss_opt_problem}. 

\input{algs/training_alg}

%% file: algs/evaluation_alg.tex
\begin{algorithm}[t]
\caption{Evaluating \(\TLneuralODE_{\nnParams, \nnParamsMidpoint}(\initState, \initTime, \predictionTime)\)}
\label{alg:evaluation_alg}
\textbf{Input}: \(\initState\), \(\initTime\), \(\predictionTime\)\\
\textbf{Parameter}: \(\nnParams\), \(\nnParamsMidpoint\), \(\expansionOrder\), \(\numIntervals\)\\
\textbf{Output}: Model prediction \(\hat{\state}_{\predictionTime}\).

\begin{algorithmic}[1] %[1] enables line numbers

\STATE \(\hat{\state}_{\timeval_0} \gets \initState\); \(\timestep \gets \frac{\predictionTime - \initTime}{\numIntervals}\) 
\STATE \textbf{for} \(i=0,1,\ldots, \numIntervals\) \textbf{do\{} \(t_i \gets \initTime + i \timestep\)\textbf{\}}

\FOR{\(i = 0, 1, \ldots, \numIntervals - 1\)}
    \STATE \(\midpoint \gets \hat{\state}_{\timeval_i} + \bar{\midpoint}_{\nnParamsMidpoint}(\hat{\state}_{\timeval_i}, \timestep) \odot \vectorField_{\nnParams}(\hat{\state}_i)\)
    \STATE \(\hat{\state}_{t_{i+1}} \gets \hat{\state}_{t_i} + \sum_{l=1}^{p-1}\timestep^l \vectorField_{\nnParams}^{[l]}(\hat{\state}_{i}) + \timestep^p \vectorField_{\nnParams}^{[p]}(\midpoint)\)
\ENDFOR

\RETURN \(\hat{\state}_{\numIntervals}\)

\end{algorithmic}

\end{algorithm}

%% file: algs/training_alg.tex
\begin{algorithm}[t]
\caption{Training the TL-NODE}
\label{alg:training_algorithm}
\textbf{Input}: Training dataset $\dataset$\\
\textbf{Parameter}: \(N_{\nnParams}\), \(N_{\nnParamsMidpoint}\), \(N_{train}\), \(N_{|\dataset_{\nnParamsFixed}|}\)\\
\textbf{Output}: Model parameters \(\nnParams\), \(\nnParamsMidpoint\)

\begin{algorithmic}[1] %[1] enables line numbers

\STATE Initialize parameters \(\nnParams\), \(\nnParamsMidpoint\)

\FOR{\(N_{train}\) steps} 

    \STATE Fix \(\nnParamsMidpointFixed \gets \nnParamsMidpoint\)
    \STATE \textbf{for} \(N_{\nnParams}\) steps \textbf{do\{} \(\nnParams \gets \mathrm{sgdStep}(\text{Eq. } \eqref{eq:loss_opt_problem}, \nnParams, \nnParamsMidpointFixed, \dataset)\)\textbf{\}}
    
    \STATE Fix \(\nnParamsFixed \gets \nnParams\); \(\{(\initState^j, \initTime^j, \predictionTime^j)\}_j \gets \mathrm{Sample}(\dataset, N_{|\dataset_{\nnParamsFixed}|})\)
    \STATE \(\dataset_{\nnParamsFixed} \gets \mathrm{ODESolve}(\vectorField_{\nnParamsFixed}, \{(\initState^j, \initTime^j, \predictionTime^j)\}_j)\)
    \STATE \textbf{for} \(N_{\nnParamsMidpoint}\) steps \textbf{do\{} \(\nnParamsMidpoint \gets \mathrm{sgdStep}(\text{Eq. } \eqref{eq:midpoint_loss_opt_problem}, \nnParamsFixed, \nnParamsMidpoint, \dataset_{\nnParamsFixed})\)\textbf{\}}

\ENDFOR

\RETURN \(\nnParams, \nnParamsMidpoint\)

\end{algorithmic}

\end{algorithm}

%% file: tex/05_numerical_experiments.tex
\section{Numerical Experiments}
\label{sec:numerical_experiments}

\input{table/mnist_table}

We demonstrate the effectiveness of TL-NODE through several numerical experiments: the numerical integration of known dynamics, the learning of unknown dynamics, a supervised classification task, and a density estimation task.
As an initial illustrative example we apply TL-NODE to linear dynamics. However, we note that the latter classification and density estimation tasks involve non-linear, time-dependent, and high-dimensional dynamics. 
Additional experimental details -- including hyperparameter selection -- are included in Appendix \ref{sec:appdix_experimental_details}.

\begin{figure}[t]
\centering
\input{figs/stiff_dynamics/plot_integration_known_dynamics}
\vspace{-5mm}
\caption{Numerical integration of known stiff dynamics. 
All algorithms other than Dopri5 use $H=1$.
Top: average time for numerical integration, as a function of the size of the prediction time interval. Bottom: average normalized integration error. Averages are taken with respect to 250 randomly sampled initial states.}
\vspace{-5mm}
\label{fig:known_dynamics_results}
\end{figure}
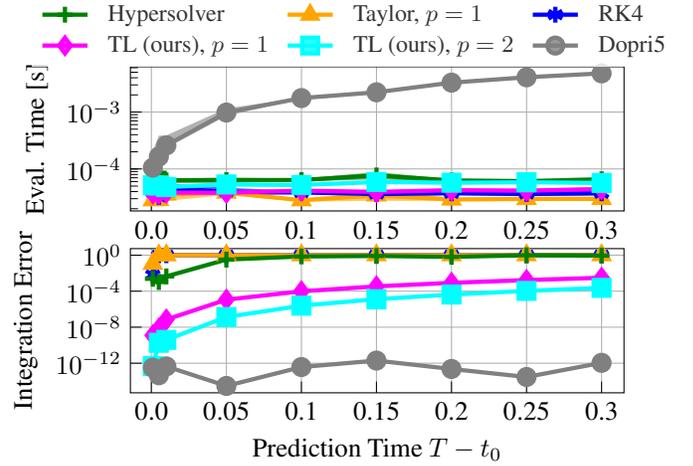

\subsection{Modeling a Dynamical System}
\label{sec:modeling_dynamical_system}

We begin by applying TL-NODE to the task of modeling a stiff dynamical system.
More specifically, we use the proposed Taylor-Lagrange approach to learn, and to integrate, the ODE \(\dot{\state} = A \state\), where \(\state \in \mathbb{R}^2\) and \(A \in \mathbb{R}^{2 \times 2}\) has eigenvalues \(\lambda_1 = -1\) and \(\lambda_2 = -1000\).

\subsubsection{Integration of Known Stiff Dynamics.}

To examine the accuracy and robustness of the midpoint prediction network \(\midpoint_{\nnParamsMidpoint}(\cdot)\), we begin by assuming the dynamics function \(\vectorField(\state) = A \state\) is known, and we use the proposed Taylor-Lagrange numerical integration method to predict future system states.
We note that because we assume \(\vectorField(\cdot)\) is known, there is no need to parameterize the system dynamics using a neural network \(\vectorField_{\nnParams}(\cdot)\).
However, we may still apply the method outlined in \S \ref{sec:training_tlnode} to train \(\midpoint_{\nnParamsMidpoint}(\cdot)\) to predict the approximation error of the Taylor-Lagrange expansions.

\textbf{Baselines.}
We apply both \(1^{st}\) and \(2^{nd}\) order Taylor expansions for numerical integration.
For comparison, we include the results of a fixed-step RK4 method, an adaptive-step method (Dopri5), and the \textit{Hypersolver} method \cite{Poli2020HypersolversTF}.
The tolerance parameters \(rtol\) and \(atol\) of the adaptive-step Dopri5 integrator are both set to $1.4e^{-12}$.
We also plot the result of using a Taylor expansion for integration, without including the learned approximation error term.

\textbf{Results.} Figure \ref{fig:known_dynamics_results} illustrates the numerical integration results.
For brevity, in the figure we refer to the proposed Taylor-Lagrange method for integration as TL.
We observe that TL-NODE enjoys lower integration error than all of the baseline methods except for Dopri5.
However, Dopri5 requires computation times that are more than an order of magnitude higher than that of our method.
We additionally observe that while the Hypsersolver method requires similar computation time to TL-NODE, the error of its numerical integration results are several orders of magnitude higher.
Furthermore, we note that for any prediction time intervals \(\predictionTime - \initTime\) larger than \(0.05 (s)\), the fixed-step RK4, Truncated Taylor expansion method, and Hypersolver method all have normalized prediction errors values of \(1.0\) (the highest possible value).
By contrast, our TL-NODE approach achieves an average error value on the order of \(10^{-4}\), even when \(\predictionTime - \initTime = 0.3(s)\).
This demonstrataes the robustness of the proposed approach to the size of the prediction interval.

Finally, we note that the integration error of the truncated Taylor expansion method (yellow) is several orders of magnitude larger than that of TLN.
The only difference between these methods is TLN's inclusion of the proposed correction term that learns the approximation error, demonstrating the gain in accuracy that this learned correction term provides.

\subsubsection{Learning Unknown Dynamics}

We now assume that the system dynamics \(\vectorField(\state) = A \state\) are unknown and train a TL-NODE to model the dynamical system.

\textbf{Baselines.}
We compare to Vanilla NODE, which uses adaptive-step Dopri5 for numerical integration, to a NODE trained using fixed-step RK4, and to T-NODE -- a version of our approach that also uses Taylor expansions for integration, but does not estimate their remainder term.

\textbf{Results.}
Figure \ref{fig:learn_dynamics_results} illustrates the NODE's average prediction error as a function of the number of elapsed training steps. TL-NODE achieves smiliar prediction error values to the Vanilla NODE throughout training, while the prediction errors of the other two baseline methods are twice as large.
The wall-clock training time for TL-NODE is 31.9s, for the Vanilla NODE it is 609.8s, for the RK4 NODE it is 35.8s, and for T-NODE it is 21.2s.
Algorithm 2 effectively balances the training of TL-NODE's two components: \(\vectorField_{\nnParams}(\cdot)\) and \(\midpoint_{\nnParamsMidpoint}(\cdot)\). The result is a dynamics model that is as accurate as the Vanilla NODE trained using Dopri5, but whose training and evaluation times are much faster.

\begin{figure}[t]
\centering
\input{figs/plot_learn_dynamics}
\vspace{-5mm}
\caption{
    Predicting unknown dynamics over a prediction time step of \(\predictionTime - \initTime = 0.01 (s)\).
    We plot the average mean square error of the predicted state as a function of the elapsed training steps.
    }
\vspace{-2mm}
\label{fig:learn_dynamics_results}
\end{figure}
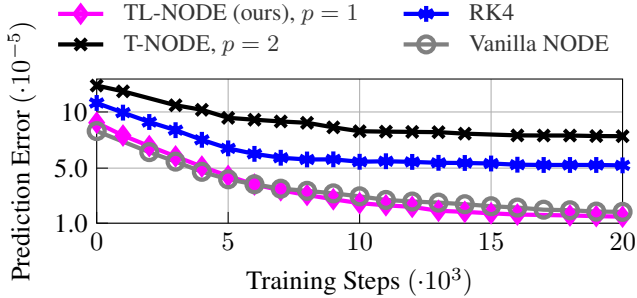

%% file: table/mnist_table.tex
\begin{table*}[t]
  \centering
  \begin{tabular}{@{}lllllll@{}}
  \hline
    \textbf{Method} & \textbf{Train Accuracy (\%)} & \textbf{Test Accuracy (\%)} & \textbf{Train Time (min)} & \textbf{Eval. Time (ms)} & \textbf{NFE} \\
    \hline
    TL-NODE (ours) & 99.96 & \textbf{98.23} & \textbf{2.55} & \textbf{1.04} & \textbf{62}  \\
    Vanilla NODE & 99.33 & 97.87 & 42.7 & 16 & 110.6 \\
    TayNODE & 99.29 & 98.02 & 94.3 & 11 & 80.00 \\
    RNODE & 98.72 & 97.74 & 10.2 & 2.05 & 98.0 \\
    SRNODE* & \textbf{100.0} & 98.08 & 98.1 & - & 259.0 \\
    STEER* & \textbf{100.0} & 97.94 & 103 & - & 265.0 \\
    \hline
  \end{tabular}
  \caption{MNIST image classification results.
  }
  \label{tab:mnist}
\end{table*}

%% file: figs/stiff_dynamics/plot_integration_known_dynamics.tex
% This file was created with tikzplotlib v0.9.12.
\begin{tikzpicture}

\definecolor{color0}{rgb}{1,0.647058823529412,0}
\definecolor{color1}{rgb}{1,0,1}
\definecolor{color2}{rgb}{0,1,1}

\begin{groupplot}[group style={group name = plots, group size=1 by 2, vertical sep=0.5cm, horizontal sep=0.0cm}]

\nextgroupplot[
width=0.95\columnwidth, 
height=3.5cm,
tick align=inside,
legend cell align={left},
legend columns=3,
legend style={
    font=\footnotesize,
    anchor=south west,
    at={(-0.2, 1.0)},
    fill opacity=0.8, 
    draw opacity=1, 
    text opacity=1,
    draw=none, 
    /tikz/every even column/.append style={column sep=0.15cm}
},
log basis y={10},
tick align=inside,
tick pos=left,
x grid style={white!69.0196078431373!black},
xmajorgrids,
xmin=-0.01395, xmax=0.31495,
xtick style={color=black},
y grid style={white!69.0196078431373!black},
ylabel={Eval. Time [s]},
ymajorgrids,
ymin=1.82885473987167e-05, ymax=0.00624767805561898,
ymode=log,
ytick style={color=black},
xtick={0.0, 0.05, 0.1, 0.15, 0.2, 0.25, 0.3},
xticklabels={0.0, 0.05, 0.1, 0.15, 0.2, 0.25, 0.3},
]
\input{figs/stiff_dynamics/cTime_dt_data}

\nextgroupplot[
width=0.95\columnwidth,
height=3.5cm,
tick align=inside,
legend cell align={left},
legend style={fill opacity=0.8, draw opacity=1, text opacity=1, draw=white!80!black},
log basis y={10},
tick pos=left,
x grid style={white!69.0196078431373!black},
xlabel={Prediction Time \(\predictionTime - \initTime\)},
xmajorgrids,
xmin=-0.01395, xmax=0.31495,
xtick style={color=black},
y grid style={white!69.0196078431373!black},
ylabel={Integration Error}, % Mean for time step 0.001},
ymajorgrids,
ymin=5.14253098287814e-16, ymax=5.34612604750051,
ymode=log,
ytick style={color=black},
xtick={0.0, 0.05, 0.1, 0.15, 0.2, 0.25, 0.3},
xticklabels={0.0, 0.05, 0.1, 0.15, 0.2, 0.25, 0.3},
ytick={1e-12,1e-8, 1e-4, 1.0},
yticklabels={\(10^{-12}\),\(10^{-8}\),\(10^{-4}\), \(10^0\)},
]

\input{figs/stiff_dynamics/accuracy_dt_data}

\end{groupplot}

\end{tikzpicture}

%% file: figs/plot_learn_dynamics.tex
% This file was created with tikzplotlib v0.9.12.
\begin{tikzpicture}

\definecolor{color0}{rgb}{1,0.647058823529412,0}
\definecolor{color1}{rgb}{1,0,1}

\begin{axis}[
height=3.5cm,
width=\columnwidth,
legend cell align={left},
legend columns = 2,
legend style={
    /tikz/every even column/.append style={column sep=0.5cm},
    font=\footnotesize, 
    fill opacity=0.8, 
    draw opacity=1, 
    text opacity=1, 
    draw=none,
    at={(1.0, 1.6)},
    },
tick align=inside,
tick pos=left,
x grid style={white!69.0196078431373!black},
xlabel={Training Steps \((\cdot 10^3)\)},
xmajorgrids,
xmin=0, xmax=20,
xtick style={color=black},
y grid style={white!69.0196078431373!black},
ylabel={Prediction Error \((\cdot10^{-5})\)},
ymajorgrids,
ymin=9.64840992828882e-07, ymax=0.000129321379371469,
ytick style={color=black},
ytick={0.000001, 0.00005, 0.0001},
yticklabels={1.0, 5.0, 10},
scaled y ticks=false
]

\addplot [ultra thick, color1, mark=diamond*, mark size=3, mark options={solid}]
table {%
0 9.0479850769043e-05
1 7.9035758972168e-05
2 6.96182250976562e-05
3 6.04391098022461e-05
4 5.13792037963867e-05
5 4.3034553527832e-05
6 3.5405158996582e-05
7 3.02791595458984e-05
8 2.61068344116211e-05
9 2.22921371459961e-05
10 1.88350677490234e-05
11 1.69277191162109e-05
12 1.54972076416016e-05
13 1.20401382446289e-05
14 1.10864639282227e-05
15 9.77516174316406e-06
16 8.70227813720703e-06
18 7.86781311035156e-06
19 7.27176666259766e-06
20 6.79492950439453e-06
};
\addlegendentry{TL-NODE (ours), \(p=1\)}
% \addplot [thick, color1, forget plot]
% table {%
% 0 9.10758972167969e-05
% 1 8.13007354736328e-05
% 2 7.15255737304688e-05
% 3 6.23464584350586e-05
% 4 5.43594360351562e-05
% 5 4.58955764770508e-05
% 6 3.86238098144531e-05
% 7 3.40938568115234e-05
% 8 3.01599502563477e-05
% 9 2.65836715698242e-05
% 10 2.30073928833008e-05
% 11 2.11000442504883e-05
% 12 1.95503234863281e-05
% 13 1.6331672668457e-05
% 14 1.68085098266602e-05
% 15 1.62124633789062e-05
% 16 1.50203704833984e-05
% 17 1.51395797729492e-05
% 18 1.57356262207031e-05
% 19 1.56164169311523e-05
% 20 1.56164169311523e-05
% };

\addplot [ultra thick, blue, mark=asterisk, mark size=3, mark options={solid}]
table {%
0 0.000108003616333008
1 9.94205474853516e-05
2 9.10758972167969e-05
3 8.38041305541992e-05
4 7.5221061706543e-05
5 6.7591667175293e-05
6 6.27040863037109e-05
7 5.92470169067383e-05
8 5.76972961425781e-05
9 5.78165054321289e-05
10 5.56707382202148e-05
11 5.6147575378418e-05
12 5.55515289306641e-05
13 5.47170639038086e-05
14 5.45978546142578e-05
15 5.41210174560547e-05
16 5.30481338500977e-05
17 5.28097152709961e-05
18 5.28097152709961e-05
19 5.28097152709961e-05
20 5.24520874023438e-05
};
\addlegendentry{RK4}
% \addplot [thick, blue, forget plot]
% table {%
% 0 0.000109195709228516
% 1 0.00010073184967041
% 2 9.2625617980957e-05
% 3 8.54730606079102e-05
% 4 7.70092010498047e-05
% 5 6.96182250976562e-05
% 6 6.47306442260742e-05
% 7 6.13927841186523e-05
% 8 5.9962272644043e-05
% 9 6.02006912231445e-05
% 10 5.8293342590332e-05
% 11 5.86509704589844e-05
% 12 5.81741333007812e-05
% 13 5.74588775634766e-05
% 14 5.73396682739258e-05
% 15 5.69820404052734e-05
% 16 5.62667846679688e-05
% 17 5.59091567993164e-05
% 19 5.57899475097656e-05
% 20 5.54323196411133e-05
% };

\addplot [ultra thick, black, mark=x, mark size=3, mark options={solid}]
table {%
0 0.000123500823974609
1 0.000118374824523926
3 0.000105977058410645
4 0.000101923942565918
5 9.47713851928711e-05
6 9.31024551391602e-05
7 9.1552734375e-05
8 9.03606414794922e-05
9 8.63075256347656e-05
10 8.2850456237793e-05
11 8.24928283691406e-05
12 8.22544097900391e-05
13 8.20159912109375e-05
14 8.07046890258789e-05
16 7.91549682617188e-05
17 7.9035758972168e-05
18 7.9035758972168e-05
19 7.85589218139648e-05
20 7.84397125244141e-05
};
\addlegendentry{T-NODE, \(p=2\)}
% \addplot [thick, color0, forget plot]
% table {%
% 0 0.000117897987365723
% 1 0.000112533569335938
% 3 9.97781753540039e-05
% 4 9.56058502197266e-05
% 5 8.82148742675781e-05
% 6 8.6665153503418e-05
% 7 8.53538513183594e-05
% 8 8.44001770019531e-05
% 9 8.14199447631836e-05
% 10 7.87973403930664e-05
% 11 7.84397125244141e-05
% 12 7.83205032348633e-05
% 13 7.80820846557617e-05
% 14 7.70092010498047e-05
% 16 7.58171081542969e-05
% 17 7.56978988647461e-05
% 18 7.56978988647461e-05
% 19 7.5221061706543e-05
% 20 7.5221061706543e-05
% };

\addplot [ultra thick, white!50.1960784313725!black, mark=o, mark size=3, mark options={solid}]
table {%
0 8.29696655273438e-05
2 6.42538070678711e-05
3 5.57899475097656e-05
4 4.67300415039062e-05
5 4.00543212890625e-05
6 3.56435775756836e-05
7 3.17096710205078e-05
8 2.98023223876953e-05
9 2.75373458862305e-05
10 2.49147415161133e-05
11 2.16960906982422e-05
12 2.02655792236328e-05
13 1.93119049072266e-05
14 1.78813934326172e-05
15 1.60932540893555e-05
16 1.49011611938477e-05
17 1.29938125610352e-05
18 1.2516975402832e-05
19 1.13248825073242e-05
20 1.10864639282227e-05
21 1.06096267700195e-05
22 1.02519989013672e-05
24 9.89437103271484e-06
25 9.41753387451172e-06
27 9.05990600585938e-06
28 8.70227813720703e-06
31 8.22544097900391e-06
32 8.10623168945312e-06
34 7.86781311035156e-06
35 7.74860382080078e-06
37 7.51018524169922e-06
40 7.39097595214844e-06
};
\addlegendentry{Vanilla NODE}
\end{axis}

\end{tikzpicture}

%% file: tex/06_supervised_classification.tex
\subsection{Supervised Classification}
\label{sec:supervised_classification}

We train a TL-NODE model to perform classification on the MNIST dataset \cite{deng2012mnist}. 
Our model follows the architecture of the neural ODE-based MNIST classifier presented in the work of~\cite{kelly2020learning} and further used in~\cite{pmlr-v139-pal21a} for benchmarking. Specifically, the model uses a two-layered neural network of size $100$ and $728$ (size of the images) with sigmoid-based non linearities to parameterize the dynamics function \(\vectorField_{\nnParams}(\cdot)\). The NODE outputs propagate through a linear classifier to estimate of the image labels.

\textbf{Baselines.}
We compare the proposed Taylor-Lagrange networks with state-of-the-art NODE algorithms. 
More specifically, we compare to RNODE~\cite{Finlay2020HowTT} and TayNode~\cite{kelly2020learning} using the source code provided by the respective authors. 
We implicitly also compare our results with other regularization techniques such STEER~\cite{ghosh2020steer} and SRNODE~\cite{pmlr-v139-pal21a} thanks to the thorough experiments provided in \cite{pmlr-v139-pal21a} for a similar model of the MNIST classification problem.

\textbf{Results.}
Table~\ref{tab:mnist} lists the experimental results. 
TL-NODE achieves evaluation and training times that are more than an order of magnitude faster than the baseline approaches, while also achieving the highest accuracy on the test dataset. TL-NODE also learns a dynamics network \(\vectorField_{\nnParams}(\cdot)\) that requires the smallest number of function evaluations (NFE) when it is being numerically integrated using an adaptive-step integrator.
The low NFE score of TL-NODE indicates that the regularization term in \eqref{eq:loss_opt_problem} is effective at producing learned dynamics networks \(\vectorField_{\nnParams}(\cdot)\) that are easy to numerically integrate.

%% file: tex/07_density_estimation.tex
\subsection{Density Estimation}
\label{sec:density_estimation_tasks}
We apply TL-NODEs to train continuous-normalizing-flow-based generative models \cite{chen2018neural,grathwohl2018ffjord} to approximate the distribution of the MiniBooNE dataset \cite{roe2005boosted,papamakarios2017masked}.
\input{table/miniboone_table}
\vspace{-3mm}
\textbf{Results.}
TL-NODE achieves the best loss score and the lowest required number of function evaluations (NFE) in comparison with the baseline approaches.

%% file: table/miniboone_table.tex
\begin{table}[H]
  \centering
  \begin{tabular}{@{}llll@{}}
  \hline
    \textbf{Method} & \textbf{Loss (nat)} & \textbf{Train Time (min)} & \textbf{NFE}\\
    \hline
    TL-NODE (ours) & \textbf{9.62} & 12.3 & \textbf{167.9}\\
    Vanilla NODE & 9.74  & 59.7 & 183.8\\
    TayNODE & 9.75 & 148.3 & 168.2\\
    RNODE & 9.78 & \textbf{10.32} & 182.0\\
    \hline
  \end{tabular}
  \caption{Density estimation results.}
  \label{tab:miniboone}
\end{table}

%% file: tex/08_conclusions.tex
\section{Conclusions}

We present Taylor-Lagrange Neural Ordinary Differential Equations (TL-NODEs): a class of neural ODEs (NODEs) that use fixed-order Taylor expansions for numerical integration during NODE training and evaluation. 
TL-NODEs also train a separate neural network to predict the expansion's remainder, which is used as a correction term to improve the accuracy of the NODE's outputs.
We demonstrate that TL-NODEs enjoy evaluation and training times that are an order of magnitude faster than the current state-of-the-art, without any loss in accuracy.
Future work will aim to apply the accelerated NODE evaluation times to the online model-based control of unknown dynamical systems.

%% file: tex/09_proofs.tex
\setcounter{theorem}{0}

\section{Proof of Theorem \ref{thm:supp-expression-midpoint}}\label{sec:appendix_proof_thm_1}

The proof of this theorem is adapted from the results provided in Theorem $2$ of~\cite{djeumou2021fly}.

We assume that the dynamics function $\vectorField_{\nnParams}$ is locally Lipschitz continuous. Indeed, most common activation functions enforce Lipschitz continuity of the neural network encoding $\vectorField_{\nnParams}$. Given that the state $\state$ lies in a domain $\mathcal{X}\subseteq \mathbb{R}^n$, The Lipschitz assumption provides that there exists a constant $L_{\vectorField_{\nnParams}} \in \mathbb{R}^n$ such that 
$$(L_{\vectorField_{\nnParams}})_k = \sup \{ L \in \mathbb{R} \: \vert \: |(\vectorField_{\nnParams})_k(x) - (\vectorField_{\nnParams})_k(y)| \leq L \|x-y\|_2, x,y \in \mathcal{X}, x \neq y\}, \; \text{for all } k \in [1,\hdots, n],$$
where $(\vectorField_{\nnParams})_k$ and $(L_{\vectorField_{\nnParams}})_k$ denote the $k$-th component of the vectors $\vectorField_{\nnParams}$ and $L_{\vectorField_{\nnParams}}$ respectively.
We use such Lipschitz assumption to provide the proof of Theorem \ref{thm:supp-expression-midpoint} in two steps.

First, we provide a bound on the rate of changes of solutions $x(t)$ of the differential equation $\dot{x}(t) = \vectorField_{\nnParams}(x(t))$.

\begin{lemma}[\textsc{Variation of trajectories of $\dot{x}(t) = \vectorField_{\nnParams}(x(t))$}]\label{lem:traj-var}
Let $x: \mathbb{R}_+ \mapsto \mathcal{X}$ be a continuous-time signal solution of the differential equation $\dot{x}(t) = \vectorField_{\nnParams}(x(t))$ from the initial state $x(t_i) = x_{t_i}$. Then, for all $t \in [t_i, t_{i+1}=t_i + \Delta t]$, we have that
\begin{align}
    \| x(t) - x_{t_i}\|_2 \leq \| \vectorField_{\nnParams}(x_{t_i})\|_2 \|L_{\vectorField_{\nnParams}}\|_2^{-1} (\mathrm{e}^{\|L_{\vectorField_{\nnParams}}\|_2 \Delta t} - 1), \label{eq:max-variation-solution},
\end{align}
\end{lemma}

\begin{proof}
    For all $t\in [t_{i}, t_{i+1]}$, we have that 
    \begin{align}
        \|x(t) - x_{t_i}\|_2 &= \|\int_{t_i}^{t} \vectorField_{\nnParams}(x(s)) ds \|_2 \label{eq:pvarsol-solution-characteristic} \\
        &\leq \int_{t_i}^{t} \|\vectorField_{\nnParams}(x(s))-\vectorField_{\nnParams}(x_{t_i})\|_2 ds +  \int_{t_i}^{t} \|\vectorField_{\nnParams}(x_{t_i})\|_2 \label{eq:pvarsol-triangle-inequality}\\
        &\leq \int_{t_i}^{t} \|L_{\vectorField_{\nnParams}}\|_2 \|x(s)-x_{t_i}\|_2 ds + \int_{t_i}^{t} \|\vectorField_{\nnParams}(x_{t_i})\|_2 ds \label{eq:pvarsol-lipschitz-bounds} \\
        &= \| \vectorField_{\nnParams}(x_{t_i})\|_2 (t-t_i) + \int_{t_i}^t \|L_{\vectorField_{\nnParams}}\|_2 \|x(s)-x_{t_i}\|_2 ds. \label{eq:pvarsol-pre-gronwall} 
    \end{align}
    We obtain ~\eqref{eq:pvarsol-solution-characteristic} since $x$ is a solution of $\dot{x}(t) = \vectorField_{\nnParams}(x(t))$. The passage from~\eqref{eq:pvarsol-solution-characteristic} to~\eqref{eq:pvarsol-triangle-inequality} results from applying the triangle inequality. We use The definition of the Lipschtiz bound to obtain~\eqref{eq:pvarsol-lipschitz-bounds}. Finally we will use the Grönwall's inequality to transform \eqref{eq:pvarsol-pre-gronwall} into \eqref{eq:max-variation-solution}. Specifically, the Grönwall's inequality is given as follows.  Let $\Psi, \alpha, \gamma$ be real-valued functions on $[t_0, T]$ with $t_0,T \in \mathbb{R}$. Suppose $\Psi$ satisfies, for all $t \in [t_0,T]$, the inequality
    \begin{align}
        \Psi (t) \leq \alpha (t) + \int_{t_0}^t \gamma (s) \Psi (s) ds, \label{eq:cond-gronwall-inequality}
    \end{align}
    with $\gamma (s) \geq 0 $ for all $s \in [t_0,T]$. Then, for all $t \in [t_0,T]$,
    \begin{align} \label{eq:gronwall-inequality}
        \Psi (t) \leq \alpha (t) + \int_{t_0}^{t} \alpha (s) \gamma (s) \exp \Big( \int_{s}^t \gamma (r) dr \Big) ds.
    \end{align}
    
    The inequality~\eqref{eq:pvarsol-pre-gronwall} satisfies the condition~\eqref{eq:gronwall-inequality} of the Grönwall's inequality with $\Psi (t) = \|x(t) - x_{t_i}\|_2$, $\alpha (t) = \| \vectorField_{\nnParams}(x_{t_i})\|_2 (t-t_i)$, and $\gamma(t) = \|L_{\vectorField_{\nnParams}}\|_2$. Thus, we have that
    \begin{align}
        \|x(t) - x_{t_i}\|_2 \leq \begin{aligned}[t]
                                    &\| \vectorField_{\nnParams}(x_{t_i})\|_2 (t-t_i) + \|L_{\vectorField_{\nnParams}}\|_2 \| \vectorField_{\nnParams}(x_{t_i})\|_2 \int_{t_i}^t (s-t_i) \mathrm{e}^{\|L_{\vectorField_{\nnParams}}\|_2(t-s)} ds.
                                \end{aligned} \label{eq:pvarsol-post-gronwall}
    \end{align}
    By integration by parts, we have that
    \begin{align}
        \int_{t_i}^t (s-t_i) \mathrm{e}^{\|L_{\vectorField_{\nnParams}}\|_2(t-s)} ds &= - \frac{t-t_i}{\|L_{\vectorField_{\nnParams}}\|_2} + \frac{1}{\|L_{\vectorField_{\nnParams}}\|_2}\int_{t_i}^t \mathrm{e}^{\|L_{\vectorField_{\nnParams}}\|_2(t-s)} ds = -\frac{t-t_i}{\|L_{\vectorField_{\nnParams}}\|_2} + \frac{\mathrm{e}^{\|L_{\vectorField_{\nnParams}}\|_2 (t-t_i)} - 1}{\|L_{\vectorField_{\nnParams}}\|_2^2}. \label{eq:pvarsol-integration-part}
    \end{align}
    Finally, by combining~\eqref{eq:pvarsol-integration-part} and~\eqref{eq:pvarsol-post-gronwall}, we obtain~\eqref{eq:max-variation-solution}.
\end{proof} 

Next, we obtain the result in Theorem~\ref{thm:supp-expression-midpoint} by exploiting the following closed-form expression of an a priori enclosure $\mathcal{S}_{i+1} \subseteq \mathcal{X}$ of the system's state $x(t_{i+1})$ at time $t_{i+1}$. That is, $x(t_{i+1}) \in \mathcal{S}_{i+1}$. Indeed, the expression below obtained for $\mathcal{S}_{i+1}$ provide insight on how to compute the Midpoint in a manner that incorporate information on the underlying dynamics $\vectorField_{\nnParams}$.
\begin{lemma}[\textsc{A priori Enclosure Estimation}]\label{lem:apriori-encl}
    Given the Lipschitz continuous function $\vectorField_{\nnParams}$, the Lipschitz bound $L_{\vectorField_{\nnParams}}$, and an initial state $x(t_i) = x_{t_i}$, then under the assumption that $\Delta t \sqrt{n} \|L_{\vectorField_{\nnParams}}\|_2 < 1 $, an a priori rough enclosure $\mathcal{S}_{i+1} \subseteq \mathcal{X}$ of $x(t_i + \Delta t)$ is given by
    \begin{align}
        x(t_i + \Delta t) \in \mathcal{S}_i = x_{t_i} + [-1,1]^{n \times n}\frac{\Delta t \vectorField_{\nnParams}(x_{t_i})}{1- \sqrt{n} \Delta t \|L_{\vectorField_{\nnParams}}\|_2 }, \label{eq:explict-fixpoint}
    \end{align}
    where $[-1,1]^{n \times n}$ is a set matrix of size $n \times n$ with elements between $-1$ and $1$. 
\end{lemma} 
\begin{proof}
    First, it has been proved that a set $\mathcal{S}_i$ satisfying the fixed-point equation $x_{t_i} + [0, \Delta t] \mathrm{Im}(\vectorField_{\nnParams}, \mathcal{S}_i) \subseteq \mathcal{S}_i$ is such that the solution of the differential equation at $t_i + \Delta t$ satisfies $x(t_i + \Delta t) \in \mathcal{S}_i$. Here, all the operations are set-based operations, $[0,\Delta t]$ is the interval of values between $0$ and $\Delta t$, and $\mathrm{Im}(\cdot, \cdot)$ represents the range of a function over a given domain. In the remainder of this proof, all sets are intervals outer approximations and the operations between sets are propagated through interval arithmetic. We refer the reader to the work by~\cite{djeumou2021fly} for more details and references on the a priori rough enclosure.
    
    The expression~\eqref{eq:explict-fixpoint} is derived by scaling adequately the bound~\eqref{eq:max-variation-solution} from Lemma~\ref{lem:traj-var}. Specifically, we seek for an a priori rough enclosure $\mathcal{S}_i$ such that 
    \begin{equation}\label{eq:fixpoint-to-find}
        \mathcal{S}_i = x_{t_i} + \mu \frac{ \| \vectorField_{\nnParams}(x_{t_i})\|_2}{\|L_{\vectorField_{\nnParams}}\|_2} (\mathrm{e}^{\|L_{\vectorField_{\nnParams}}\|_2 \Delta t} - 1) [-1,1]^n,
    \end{equation}
    where $\mu > 0$ is a parameter to find in order for $\mathcal{S}_i$ to satisfy the fixed-point equation. 
    
    We over-approximate the set $\mathrm{Im}(\vectorField_{\nnParams}, \mathcal{S}_i)$ as a function of $\Delta t$, $\|L_{\vectorField_{\nnParams}}\|_2$. Specifically, for all $s_i \in \mathcal{S}_i$, we have that
    \begin{align}
        \|\vectorField_{\nnParams}(s_i) - \vectorField_{\nnParams}(x_{t_i})\| \leq  \|L_{\vectorField_{\nnParams}}\|_2\|s_i - x_{t_i}\|
        \leq  \sqrt{n} \mu \| \vectorField_{\nnParams}(x_{t_i})\|_2 (\mathrm{e}^{\|L_{\vectorField_{\nnParams}}\|_2 \Delta t} - 1). \label{eq:pfixpoint-bound-h}
    \end{align}
\end{proof}
    Here, the definition of $\mathcal{S}_i$ in ~\eqref{eq:fixpoint-to-find} provides an upper bound on $\|s_i - x_i\|$ that yields~\eqref{eq:pfixpoint-bound-h}. Additionally, the inequality~\eqref{eq:pfixpoint-bound-h} implies that
    
    \begin{equation*}
         \mathrm{Im}(\vectorField_{\nnParams}, \mathcal{S}_i ) \subseteq \vectorField_{\nnParams}(x_{t_i}) + \sqrt{n} \mu \| \vectorField_{\nnParams}(x_{t_i})\|_2 (\mathrm{e}^{\|L_{\vectorField_{\nnParams}}\|_2 \Delta t} - 1) [-1,1]^n.
    \end{equation*}
    Hence, $\mathcal{S}_i$ from~\eqref{eq:fixpoint-to-find} solves the fixed-point equation if
    \begin{equation}
        [0, \Delta t] \big( \vectorField_{\nnParams}(x_{t_i}) + \sqrt{n} \mu \| \vectorField_{\nnParams}(x_{t_i})\|_2 (\mathrm{e}^{\|L_{\vectorField_{\nnParams}}\|_2 \Delta t} - 1) [-1,1]^n \big) \subseteq  \mu \frac{ \| \vectorField_{\nnParams}(x_{t_i})\|_2}{\|L_{\vectorField_{\nnParams}}\|_2} (\mathrm{e}^{\|L_{\vectorField_{\nnParams}}\|_2 \Delta t} - 1) [-1,1]^n. \label{eq:pfixpoint-rincl}
    \end{equation}
    For notation brevity, let $c_1 = \| \vectorField_{\nnParams}(x_{t_i})\|_2 (\mathrm{e}^{\|L_{\vectorField_{\nnParams}}\|_2 \Delta t} - 1)$ and $\beta = \|L_{\vectorField_{\nnParams}}\|_2$. Observe that with interval arithmetic, $[0,\Delta t] [a,b] = [\min (0,a), \max (0,b)]$. We use the observation to find $\mu >0$ such that the inclusion~\eqref{eq:pfixpoint-rincl} holds. That is, the inequalities 
     \begin{align*}
        \Delta t \Big( {(\vectorField_{\nnParams}(x_{t_i}))}_k + \sqrt{n} \mu c_1 \Big) \leq \frac{\mu c_1}{\beta} 
        \Longleftrightarrow (\frac{1}{\Delta t \beta} - \sqrt{n}) \mu \geq \frac{1}{c_1} {(\vectorField_{\nnParams}(x_{t_i}))}_k
    \end{align*}
    and 
    \begin{align*}
        \Delta t \Big( {(\vectorField_{\nnParams}(x_{t_i}))}_k - \sqrt{n} \mu c_1 \Big) \geq -\frac{\mu c_1}{\beta_i} \Longleftrightarrow (\sqrt{n} - \frac{1}{\Delta t \beta_i}) \mu \geq \frac{1}{c_1} {(\vectorField_{\nnParams}(x_{t_i}))}_k
    \end{align*}
    hold for all $k [1,\hdots,n]$. Therefore, for a step size $\Delta t$ satisfying $\Delta t \sqrt{n} \|L_{\vectorField_{\nnParams}}\|_2 < 1$, $\mu [-1,1]^n$ given by
    \begin{align*}
        \mu [-1, 1]^n = \frac{ [-1,1]^{n \times n} \vectorField_{\nnParams}(x_{t_i})}{c_1 (\frac{1}{\Delta t \beta_i} - \sqrt{n})}, 
    \end{align*}
    is such that $\mu$ satisfies the above inequalities. Thus, the inclusion~\eqref{eq:pfixpoint-rincl} holds and the set $\mathcal{S}_i$ is solution of the fixed-point equation. By replacing $\mu$ in $\mathcal{S}_i$ given by~\eqref{eq:fixpoint-to-find}, we obtain~\eqref{eq:explict-fixpoint}.

\begin{theorem}
    [Simplified Midpoint Expression]
    If \(f_{\theta}\) is a Lipschitz-continuous function, then there exists a function \(\bar{\Gamma} : \mathbb{R}^{n} \times \mathbb{R}_{+} \to \mathbb{R}^{n \times n}\) such that the midpoint value \(\Gamma\) of the Taylor-Lagrange expansion \(TL^{\Delta t}_{\theta, \phi}(x_{t_i})\) is related to \(x_{t_i}\), \(f_{\theta}(x_{t_i})\), and \(\bar{\Gamma}(x_{t_i}, \Delta t)\) through
    \begin{equation}
        \Gamma = x_{t_i} + \bar{\Gamma}(x_{t_i}, \Delta t) \odot  f_{\theta}(x_{t_i}),
        \label{eq:supp-midpoint-simple}
    \end{equation}
    where \(\odot\) denotes matrix-vector multiplication.
    \label{thm:supp-expression-midpoint}
\end{theorem}

\begin{proof}
    This is a direct application of the result proved in Lemma~\ref{lem:apriori-encl}. Specifically, the midpoint $\midpoint = x_\epsilon$ is a point that lies in the rough a priori enclosure $\mathcal{S}_i$, i.e.,  $\midpoint \ \in \mathcal{S}_i$,  for some $\epsilon \in [t_i, t_i + \Delta t]$. By using Lemma~\ref{lem:apriori-encl} and the expression~\eqref{eq:explict-fixpoint}, we can parameterize any point inside $\mathcal{S}_i$ as $x_{t_i} + \midpointUnknownTerm(\state_{\timeval_i}, \timestep) \odot  \vectorField_{\nnParams}(\state_{\timeval_i})$. Hence the results provided by the theorem.
\end{proof}

\section{Simplified Midpoint Expression for Linear Systems}\label{sec:appendix_proof_midpoint_linsys}
 In this section, we provide proof that, for linear systems, the function $\bar{\Gamma}$ in the simplified midpoint expression~\eqref{eq:midpoint-simple} does not depend on the state $x$. We consider linear systems in the form
 \begin{align}
     \Dot{x}(t) = A x(t), \label{eq:linear-systems}
 \end{align}
 where $x(t) \in \mathbb{R}^n$ and $A \in \mathbb{R}^{n \times n}$ is a time-independent matrix. First, by performing a zero-order Taylor-Lagrange expansion on the solution $x(t)$ of \eqref{eq:linear-systems} from an initial point $x(t_0)$, we have that
 \begin{align}
     x(t) &= x(t_0) + (t-t_0) A \Gamma \label{eq:order-0}\\
     &= x(t_0) + (t-t_0) A \Big( x(t_0) + \bar{\Gamma}(x_{t_0},  t-t_0) A x(t_0) \Big) \label{eq:simp-fixpoint}\\
     &= x(t_0) + (t-t_0) A x(t_0) + (t-t_0) \Big( A \bar{\Gamma}(x_{t_0}, t- t_0) A \Big) x(t_0), \label{eq:expanded_zeroorder}
 \end{align}
 where \eqref{eq:order-0} comes from the Taylor expansion and ~\eqref{eq:simp-fixpoint} results from the simplified midpoint expression provided in~\eqref{eq:midpoint-simple}. Next, we consider $\bar{\Gamma}(x_{t_0}, t-t_0)$ given by the state-independent expression
 \begin{align}
     \bar{\Gamma}(x_{t_0}, t-t_0) = \bar{\Gamma}(t-t_0) = \sum_{i=1}^\infty \frac{(t-t_0)^i}{(i+1)!} A^{i-1}.\label{eq:fixpoint-constant}
 \end{align}
By substituting \eqref{eq:fixpoint-constant} in \eqref{eq:expanded_zeroorder}, we obtain that
 \begin{align}
     x(t) &= x(t_0) + (t-t_0) A x(t_0) + (t-t_0)  A \Big(\sum_{i=1}^\infty \frac{(t-t_0)^i}{(i+1)!}\Big) A x(t_0)\\
     & = x(t_0) + (t-t_0) A x(t_0) + \sum_{i=2}^\infty \frac{(t-t_0)^i}{i!} A^i x(t_0) \\
     & = \exp^{A (t-t_0)} x(t_0),
 \end{align}
 where the right-hand side in the last line equation is traditionally known as the solution of the linear differential equation~\eqref{eq:linear-systems}. Thus, for the linear differential equation~\eqref{eq:linear-systems}, the function $\bar{\Gamma}$ in the simplified midpoint expression~\eqref{eq:midpoint-simple} does not depend on $x$ 

\section{Proof of Theorem \ref{supp:thm-approximation-error}}
\label{sec:appendix_proof_thm_2}

In this section, we prove Theorem~\ref{supp:thm-approximation-error} on the integration accuracy of the proposed Taylor-Lagrange expression when the remainder is approximated through deep neural networks.

\begin{theorem}[Integration Accuracy]
    If the learned midpoint predictor \(\midpoint_{\nnParamsMidpoint}(\cdot)\) is a \(\mathcal{O}(\eta)\) approximator to the midpoint \(\midpoint\) of the Taylor-Lagrange expansion of \(\taylorLagrange^{\timestep}_{\nnParams, \nnParamsMidpoint}(\state_{\timeval_i})\), then \(\| \state_{\timeval_{i+1}} - \taylorLagrange^{\timestep}_{\nnParams, \nnParamsMidpoint}(\state_{\timeval_i}) \| \leq c \eta \timestep^{p}\) for some \(c > 0\) that depends on \(\vectorField_{\nnParams}\).
    \label{supp:thm-approximation-error}
\end{theorem}

\begin{proof}
    The proof of the Theorem follows the classical proof for obtaining truncation error of Taylor expansion. Specifically, if $ \timestep^{p}\vectorField_{\nnParams}^{[p]}(\midpoint(\state_{\timeval_i}, \timestep))$ denotes the unknown midpoint function we seek to approximate, we have that 
    \begin{align}
        \| \state_{\timeval_{i+1}} - \taylorLagrange^{\timestep}_{\nnParams, \nnParamsMidpoint}(\state_{\timeval_i}) \| &=  \|  \timestep^{p}\vectorField_{\nnParams}^{[p]}(\midpoint_{\nnParamsMidpoint}(\state_{\timeval_i}, \timestep)) -  \timestep^{p}\vectorField_{\nnParams}^{[p]}(\midpoint(\state_{\timeval_i}, \timestep))\| \\
        &\leq \Delta t^p L_{\vectorField_{\nnParams}^{[p]}} \| \midpoint_{\nnParamsMidpoint}(\state_{\timeval_i}, \timestep) -  \midpoint(\state_{\timeval_i}, \timestep)\| \\
        & \leq \Delta t^p L_{\vectorField_{\nnParams}^{[p]}} \eta,
    \end{align}
    where $L_{\vectorField_{\nnParams}^{[p]}}$ is the Lipschitz constant of the $p$-th Taylor coefficient and $\eta$ is such that $\| \midpoint_{\nnParamsMidpoint}(\state_{\timeval_i}, \timestep) -  \midpoint(\state_{\timeval_i}, \timestep)\| \leq \eta$ by the assumption that \(\midpoint_{\nnParamsMidpoint}(\cdot)\) is a \(\mathcal{O}(\eta)\) approximator to the midpoint \(\midpoint\).
\end{proof} 

%% file: tex/10_experimental_details.tex
\section{Additional Experimental Details}
\label{sec:appdix_experimental_details}

\paragraph{Code.} All the implementations are written and tested in Python $3.8$, and we will release the full code upon the paper is accepted. We attach with the initial submission the latest version of the code with instructions on how to reproduce each of the results in the paper.

\paragraph{Datasets.} We use the MNIST dataset \cite{deng2012mnist} for the supervised classsification experiments and the MiniBoONe dataset \cite{roe2005boosted} for the density estimation task.

\subsection{Experiments on a Stiff Dynamical System}
In this section, we provide details on the numerical integration experiments for the stiff dynamics when the vector field is assumed to be known. Then, we provide additional details for the case where the dynamics are unknown and must be learned from data.

\subsubsection{Integrating the Known Stiff Dynamics}
We conside a $2$ dimensional linear system with dynamics given by
\(\dot{\state} = A \state\), where \(\state \in \mathbb{R}^2\) and \(A \in \mathbb{R}^{2 \times 2}\) has eigenvalues \(\lambda_1 = -1\) and \(\lambda_2 = -1000\). 
The matrix \(A\) was chosen to have such a large gap \(\lambda_1 - \lambda_2\) in order to obtain a system with stiff dynamics.
% We took $A$ with such big gap in the eigen values to obtain a stiff dynamics.

\paragraph{Training and Testing Dataset.} 
The dataset used to train \(\midpointUnknownTerm_{\nnParamsMidpoint}\) consists of $100$ separate trajectories beginning from initial states that are randomly sampled from the subset $[-0.5,0.5] \times [-0.5, 0.5]$. 
Each trajectory is integrated for \(\predictionTime=10\) seconds, for varying sizes of the integration time step \(\timestep\).
The testing dataset consists of $10$ trajectories over the same duration $T$. 
By saying that the trajectories are obtained from integration of the ODE, we mean that we use the analytic form of the solution to linear dynamics in order to obtain the testing and training trajectories.

\paragraph{TL-NODE Parameterization and Training.} We parameterize the unknown term \(\midpointUnknownTerm_{\nnParamsMidpoint}(\cdot)\) in (4) %\eqref{eq:midpoint-simple} 
as a fully connected MLP with relu activation functions and one hidden layer of size $16$.
The network is optimized using ADAM \cite{kingma2014adam} with a learning rate of $1e^{-3}$ and an exponential decay of $1e^{-4}$ per iteration. We train the neural network representing the midpoint for $1000$ epochs with a mini batch of size $512$ at each iteration.

\paragraph{Hypersolver Parameterization and Training.}
We closely follow the formulation in~\cite{Poli2020HypersolversTF} to implement the comparison with Hypersolver. Specifically, in this example, we use the HyperEuler formulation, where the solution of the ODE is given by a first order Euler approximation and the error of the approximation is given by a neural network. 
We parameterize the neural network to learn the approximation error of the Euler formulation as a fully connected MLP with relu activation functions and one hidden layer of size $32$. 
Note that this provides an advantage to TL-NODE as TL-NODE only uses $16$ nodes in the hidden layer. 
We refer the reader to ~\cite{Poli2020HypersolversTF} for more details. 

In our experiments, we observe that Hypersolver is more challenging to tune. For example, the choice of the initial values of the weights significantly impacts the algorithm convergence rate and the solution at which it converges. We did our best to tune Hypersolver such that it attains an accuracy close to the adaptive time step solver. We train the corresponding neural network for $1000$ epochs using a mini batch of size $512$. 

\subsubsection{Learning the Unknown Stiff Dynamics}
We parameterize the unknown stiff dynamics by a two-layered  neural network of size $64$ and $2$, respectively, with no activation functions. That is, the output of the neural network is a polynomial function of the states. We first train the vanilla neural ODE approach for $150$ epochs using Adam optimizer with a mini batch of size $512$. We use a learning rate of $1e-2$ and an exponential decay of $1e^{-4}$. 
We use the hyperparameters for the vanilla NODE to train RK4 and the truncated Taylor method. 

We tune our proposed approach (TL-NODE) such that it matches the performance achieved by the vanilla neural ODE. Specifically, in addition to the dynamics network \(\vectorField_{\nnParams}(\cdot)\), we parameterize the unknown term \(\midpointUnknownTerm_{\nnParamsMidpoint}(\cdot)\) in \eqref{eq:midpoint-simple} as a fully connected MLP with relu activation functions and one hidden layer of size $16$. Then, we use the same hyperparameters as for the vanilla ODE to update the dynamics network \(\vectorField_{\nnParams}(\cdot)\). We use Adam optimizer with a learning rate of $1e^{-4}$ and exponential decay of $1e^{-4}$ to simultaneously train the dynamics network \(\vectorField_{\nnParams}(\cdot)\) and the midpoint prediction network \(\midpoint_{\nnParamsMidpoint}(\cdot)\). 
We choose the parameter $N_{\phi} = 200$, which is the period at which the algorithm switches from updating the dynamics network parameters \(\nnParams\) to updating the midpoint network parameters \(\nnParamsMidpoint\) (see Algorithm~\ref{alg:training_algorithm}). $N_{\phi} = 200$ was picked to optimize the accuracy and performance of TL-NODE.

\subsection{Supervised Classification}
We train a Neural ODE and a Linear Classifier to map flattened MNIST Images~\cite{deng2012mnist} to their corresponding labels. Our model uses a two layered neural network $\vectorField_{\nnParams}$, as the ODE dynamics, followed by a linear classifier $g_\nnParams$, identical to the architecture used in~\cite{kelly2020learning}. In fact, we directly use the code provided by~\cite{kelly2020learning} to provide the comparisons shown in this paper.

\paragraph{TL-NODE Parameterization and Training}
We train our Neural ODE using a batch size of $512$ and for $1000$ epochs without any early stopping methods. We use Adam optimizer with a linear decay learning rate from $10^{-3}$ to $10^{-5}$. Additionally, we parameterize the midpoint value by a single-layer dense neural network with $24$ node. We use the penalty term value $\regularizationCoef = 2\cdot 10^2$. These values were obtained by checking how the model converges during the first $10$ epochs.

\subsection{Density Estimation}

For the models trained on the MiniBooNE tabular dataset, we used the same architecture as in Table 4 in the appendix of \cite{grathwohl2018ffjord}. Our implementation was also modified from the implementation provided by ~\cite{kelly2020learning}. The comparisons with TayNODE and RNODE are obtained from the same architecture and hyper-parameters described in~\cite{kelly2020learning}. The data was obtained as made available from \cite{papamakarios2017masked}, which was already processed and split into train/validation/test. In particular, the training set has $29556$ examples, the validation set has $3284$ examples, and the test set has 3648 examples, which consist of 43 features.

For TL-NODE, we choose the number of epochs and learning rate for both training the ODE and the midpoint such that we achieve the performance provided in~\cite{kelly2020learning}. Specifically, we parameterize the unknown term \(\midpointUnknownTerm_{\nnParamsMidpoint}(\cdot)\) in \eqref{eq:midpoint-simple} as a fully connected MLP with relu activation functions and one hidden layer of size $32$. We train the neural network encoding the ODE using Adam optimizer with a training batch size of $1000$ and a number of epochs of $400$. We pick the learning rate to be $1e-3$ for the first $300$ epochs and reduce the learning rate to $1e-5$ for the last $100$ epochs.

Then, we use Adam optimizer with a learning rate of $1e^{-4}$ and exponential decay of $1e^{-4}$ to simultaneously train the midpoint value. We choose the parameter $N_{\phi} = 50$, which is the period at which the algorithm switch from updating the ODE parameters to the midpoint parameters (see Algorithm~\ref{alg:training_algorithm}). Note that in this experiment, we need a more frequent update of the midpoint value as it helps correcting with respect to the solution of the ODE using an adaptive time step solver.